\documentclass[letterpaper]{article} % DO NOT CHANGE THIS
\usepackage{aaai25}  % DO NOT CHANGE THIS
\usepackage{times}  % DO NOT CHANGE THIS
\usepackage{helvet}  % DO NOT CHANGE THIS
\usepackage{courier}  % DO NOT CHANGE THIS
\usepackage[hyphens]{url}  % DO NOT CHANGE THIS
\usepackage{graphicx} % DO NOT CHANGE THIS
\usepackage{natbib}  % DO NOT CHANGE THIS AND DO NOT ADD ANY OPTIONS TO IT
\usepackage{caption} % DO NOT CHANGE THIS AND DO NOT ADD ANY OPTIONS TO IT
\usepackage{algorithm}
\usepackage{algorithmic}

\usepackage{newfloat}
\usepackage{listings}
\DeclareCaptionStyle{ruled}{labelfont=normalfont,labelsep=colon,strut=off} % DO NOT CHANGE THIS
\floatstyle{ruled}
\newfloat{listing}{tb}{lst}{}
\floatname{listing}{Listing}
\nocopyright
\title{Reasoning over Uncertain Text by Generative Large Language Models}
\author {
    Aliakbar Nafar\textsuperscript{\rm 1},
    Kristen Brent Venable\textsuperscript{\rm 2,\rm 3},
    Parisa Kordjamshidi\textsuperscript{\rm 1}
}
\affiliations {
    \textsuperscript{\rm 1}Michigan State University\\
    \textsuperscript{\rm 2}Florida Institute for Human and Machine Cognition\\
    \textsuperscript{\rm 3}University of West Florida\\
    \{nafarali, kordjams\}@msu.edu, bvenable@ihmc.org
}

\usepackage{fontawesome5} 
\usepackage{graphicx}

\usepackage[table]{xcolor}
\usepackage{pgf}
\usepackage{multirow}
\usepackage{multirow}
\usepackage{booktabs}
\usepackage{comment}
\usepackage{subcaption}
\usepackage{fancyvrb}
\usepackage{tabularx}
\newbox\pythonCalcBox
\newbox\monteCarloBox
\newbox\neuroSymbolicBox

\usepackage{soul}
\usepackage{url}
\usepackage{amsmath}
\usepackage{amsthm}
\usepackage{algorithm}
\usepackage{algorithmic}
\usepackage[switch]{lineno}
\usepackage{fontawesome5}
\theoremstyle{plain} 
\newtheorem{theorem}{Theorem}
\newtheorem{lemma}[theorem]{Lemma}
\usepackage{array}
\newcolumntype{P}[1]{>{\centering\arraybackslash}p{#1}}
\newcommand{\HeatCell}[1]{%
    \pgfmathsetmacro{\PercentColor}{(#1-0)*0.35}
    \xdef\PercentColor{\PercentColor}
    \cellcolor{black!\PercentColor!white}#1
}

\begin{document}

\maketitle
\begin{abstract}
This paper considers the challenges Large Language Models (LLMs) face when reasoning over text that includes information involving uncertainty explicitly quantified via probability values. This type of reasoning is relevant to a variety of contexts ranging from everyday conversations to medical decision-making. Despite improvements in the mathematical reasoning capabilities of LLMs, they still exhibit significant difficulties when it comes to probabilistic reasoning. To deal with this problem, we introduce the Bayesian Linguistic Inference Dataset (BLInD), a new dataset specifically designed to test the probabilistic reasoning capabilities of LLMs. We use BLInD to find out the limitations of LLMs for tasks involving probabilistic reasoning. In addition, we present several prompting strategies that map the problem to different formal representations, including Python code, probabilistic algorithms, and probabilistic logical programming. We conclude by providing an evaluation of our methods on BLInD and an adaptation of a causal reasoning question-answering dataset. Our empirical results highlight the effectiveness of our proposed strategies for multiple LLMs.
\end{abstract}

\section{Introduction}

Uncertainty in text is communicated in many contexts,  ranging from everyday conversations to domain-specific documents, such as those with medical focus~\cite{707988ff-79d1-31c9-afdb-b408b62c1a0f,LANDMARK201554}. Processing this uncertain information is critical. For example, uncertainty in text has been shown to significantly affect decision-making in the biomedical domain~\cite{Poggi2019UncertainWU}. Reasoning over uncertain text is also closely related to rational reasoning, e.g., if the probabilities of events A and B are low, the probability of both happening simultaneously should also be low. As a result, it is essential for language models to be able to use text with uncertainty and perform inference based on it. While the human intuitive approach to probabilistic reasoning often aligns with Bayesian Rationalism~\cite{10.1093/acprof:oso/9780198524496.001.0001,oaksford_chater_2009}, humans usually do not explicitly calculate the probabilities of outcomes. Still, probabilistic modeling using Bayesian Networks offers a robust computational approach for dealing with uncertainty. Thus, we tackle the challenge of enabling LLMs to conduct probabilistic reasoning by mapping uncertainty expressed through language to a Bayesian Network. This approach resembles other strategies for enabling mathematical reasoning over text, such as with math word problems (MWPs)~\cite{cobbe2021training,kim-etal-2023-aint}. The common theme of our problem formulation and MWPs is that a formal problem is extracted from the text and solved using external tools~\cite{ijcai2017-556,heyueya2023solving}.

\begin{figure}
    \centering
    \includegraphics[width=0.95\linewidth]{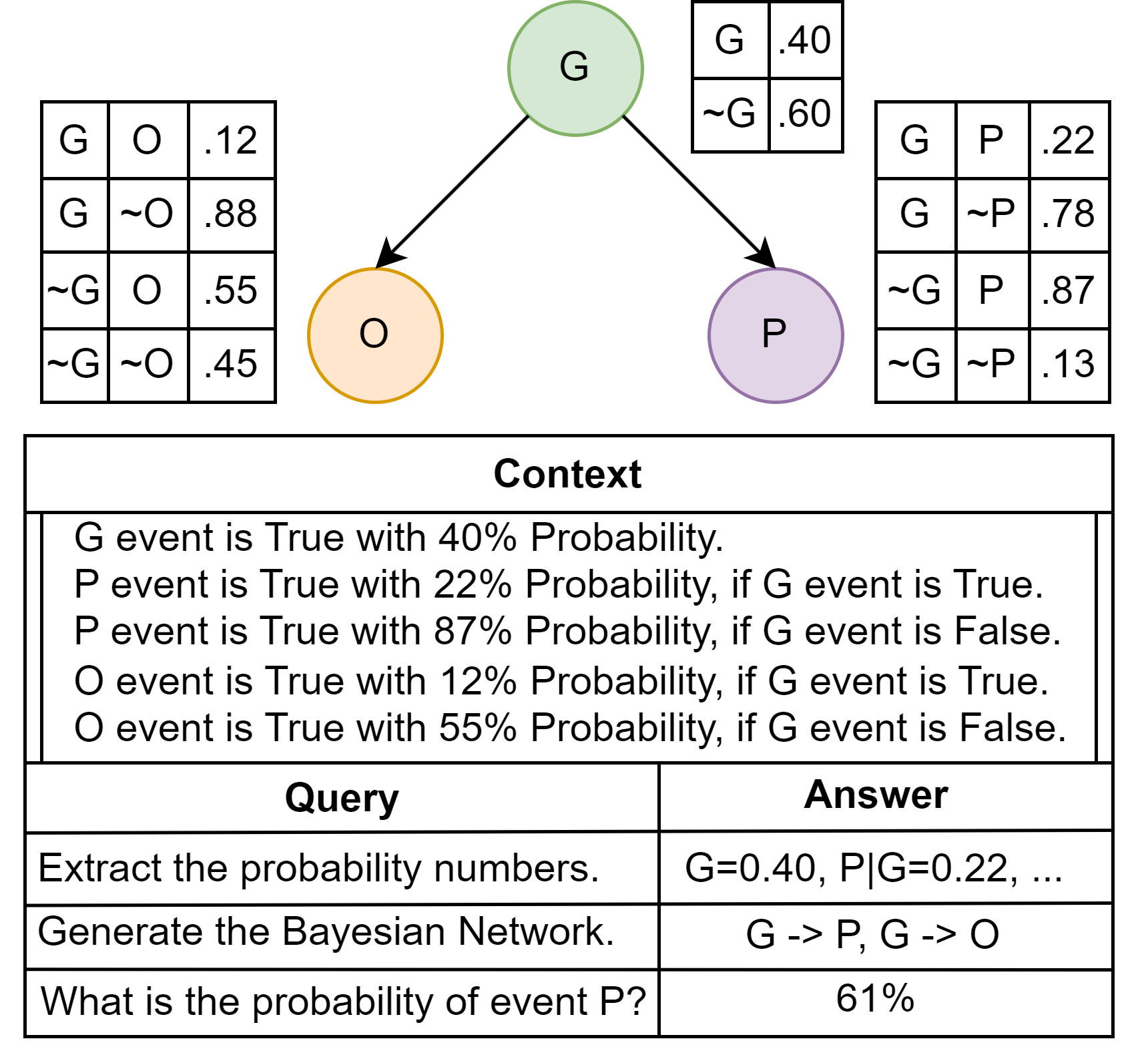}
    \caption{An example from the BLInD dataset including an underlying Bayesian network, its textual description, and probabilistic queries in natural language form.}
    \label{fig:dataset_example}
\end{figure}

First-generation LLMs were shown to struggle with mathematical reasoning and fail in answering even simple questions (e.g., about summation~\cite{mishra-etal-2022-numglue}). With the advent of newer generations of LLMs, their mathematical reasoning capability improved significantly, with GPT4 achieving 92\%~\cite{openai2023GPT4} and Gemini achieving 94.4\%~\cite{google2023google} accuracy on the Grade School Math (GSM8K) dataset~\cite{cobbe2021training}. These results have misled to the belief that LLMs are now proficient in mathematical reasoning. However, the LLMs' performance on math problems varies significantly depending on the question types~\cite{kim-etal-2023-aint}. Here, we confirm this latter result by showing that LLMs still struggle with the essential task of probabilistic reasoning over text. Furthermore, we illustrate how, depending on the LLM, different limitations and weaknesses hinder their ability to solve these problems. Simply utilizing Chain of Thought~\cite{DBLP:journals/corr/abs-2201-11903} or Python code is not always effective~\cite{shi2022stepgame,kim-etal-2023-aint}. These observations support the design of customized solutions for each problem and model. 

We focus on Bayesian inference over text and introduce a new dataset, Bayesian Linguistic Inference Dataset (BLInD), to evaluate and improve LLMs' probabilistic reasoning. BLInD instances have up to 10 interconnected random variables used to answer a probabilistic query over them. Figure~\ref{fig:dataset_example} shows a BLInD example, with the Bayesian Network and conditional probability tables at the top. The corresponding natural language explanation, which is given in input to the language models, is shown below. Given the textual context, the models are asked to answer probabilistic queries such as "What is the probability of event P?". 

We design prompts that decompose this complex problem into extracting the probability values and generating the dependency graph prior to probabilistic inference. We investigate solutions that include the above extractions as well as a mapping to symbolic solvers such as pure Python code, probabilistic inference algorithms, and probabilistic logical formalisms. Ultimately, we test our methods on our new challenging dataset and on an adapted version of a causal reasoning question-answering dataset, CLADDER~\cite{jin2023CLADDER}, further solidifying our results.

In summary, our contributions are as follows: 1) Creating a \textbf{new dataset (BLInD)}\footnote{The code and dataset are available at \url{github.com/HLR/BLInD}.} designed for reasoning over text with uncertainty explicitly quantified as probabilities; 2) \textbf{Analyzing the capabilities} of LLMs in solving the complex probabilistic reasoning problems contained in BLInD, highlighting their limitations; 3) Designing \textbf{innovative prompt engineering and in-context learning} techniques which leverage mapping to Python code, to inference algorithms, and to a probabilistic logical formalism, leading to improvements in performance across  multiple LLMs, proprietary and open-source; 4) \textbf{Evaluating} the proposed techniques on our new dataset and an adapted existing benchmark.

\section{Related Work}

A few prior works have explored question-answering (QA) involving probabilistic rules. RuleBERT~\cite{saeed-etal-2021-rulebert} and \cite{nafar-etal-2024-teaching} mainly evaluate BERT-based models~\cite{devlin-etal-2019-bert}  by fine-tuning them. They use a simple independence structure instead of dealing with arbitrary Bayesian Networks. Their queries are limited to asking the probability of a single variable, and their closed world assumption (CWA) assigns a probability of zero to any event with unspecified probability. In contrast, our queries involve any joint and marginal computation and do not use the CWA. CLADDER~\cite{jin2023CLADDER} creates a dataset with probabilistic contexts, but it is mainly designed to test the causal reasoning capabilities of LLMs with in-context learning and structured prompts. They use a limited number of variables (less than 5), their task setting is limited to binary QA, and their solution is to map the natural language text to a causal reasoning formalism. An adaptation of this dataset for mapping to probabilistic reasoning applies to our problem setting and is used in our experiments. 

Looking at reasoning over uncertain text as a form of the math word problem, \cite{ijcai2017-556,suster-etal-2021-mapping} solve simple probability word problems from introductory mathematics textbooks. However, in most questions, the probabilities are not directly given in the context, and the inference does not necessarily require mapping to Bayesian Networks. These works utilize either classical NLP parsers or fine-tuned LMs instead of our in-context prompting methods. NumGLUE~\cite{mishra-etal-2022-numglue} is the first work that analyses Pre-trained and Large Language Models for mathematical reasoning. But, it is limited to questions that require simple arithmetic reasoning. ~\cite{bubeck2023sparks,frieder2023mathematical,kim-etal-2023-aint} looks at a broader range of math questions for analyzing LLM's reasoning capabilities. However, none of these works include Bayesian probabilistic questions with complex structures.

In our solutions, we use neuro-symbolic methods to reason over uncertain text. Neuro-symbolic techniques have been used in related research to solve various NLP tasks by integration of symbolic reasoning during training or inference~\cite{rajaby-faghihi-etal-2021-domiknows, domiknows_glucons} though not for probabilistic reasoning over uncertain text. In a slightly related work, ThinkSum~\cite{ozturkler-etal-2023-thinksum} uses probabilistic reasoning by calculating the likelihood of the LLM generating each possible answer and then aggregating these token probabilities. This approach is applied to usual question answering problems that output the final crisp answers. This is fundamentally different from our work that interprets the uncertainty measures that are expressed explicitly in the text and reasons over them to infer the probability of a query.

\section{Problem Definition}

%Done
The input to the QA task is a textual \textit{context} paired with a probabilistic \textit{query} in a textual form which we refer to as the \textit{question} throughout the paper. The context comprises sentences that describe the probability of random events, which are binary variables, or the conditional probabilities of events. Figure~\ref{fig:dataset_example} shows a context with five sentences describing the probabilities of random events G, P, and O. The query can be any question that probes the probabilities of these events, such as ``What is the probability of G being true and P being false given that O is false?''. The output, which is the probability of the query, is a real number that ranges from $0.0$ to $1.0$.

\section{BLInD Generation}
\label{sec:dataset}

We introduce the Bayesian Language Inference Dataset (BLInD) to investigate the ability of LLMs to perform probabilistic reasoning. Each example in the dataset contains a textual context describing the probability of events and a textual question about the probability of events occurring in the context. Moreover, we provide a Bayesian Network (BN) corresponding to the context with their conditional probability tables (CPTs) and a probability value computed as the answer to the question. In this section, we provide an overview of the generation process and the dataset structure. Due to space constraints, details are included in the Appendix A.

\subsection{Bayesian Network}

In the first step of our dataset creation, we generate all isomorphic graphs that would serve as our Bayesian Networks, each including up to ten random variables. We generate these graphs with the following properties: 1) all graphs are \textit{Directed Acyclic (DAGs)}, 2) all graphs are \textit{Weakly Connected}, and 3) each node has at most one parent (\textit{arborescence}). This results in dataset splits denoted as $V_i$ for \( i \in \{2, 3, \ldots, 10\} \), each including a set of graphs with $i$ nodes (random variables). Properties 2 and 3 are necessary to control the complexity of the splits. The complexity increases as $i$ increases. To clarify, property 2 prevents the breakdown of a graph into smaller, independent, and subsequently simpler components. Assumption 3 results in $2+(V-1)*4$ probability entries in a BN's CPTs with $V$ variables (2 probabilities for the root node and 4 for other nodes). For example, in Figure~\ref{fig:dataset_example}, we depict a BN over nodes G (Green), O (Orange), and P (Pink) with corresponding CPTs and a total of 10 probability entries. While property 3 might restrict the networks' coverage, it enables us to analyze the examples with better control over their complexity. Further, we assume each random variable is binary and fill their associated conditional probability tables with uniformly random generated probabilities ranging from $0.01$ to $0.99$.

\subsection{Query}

We generate only \textit{complex} queries for a given Bayesian Network. By \textit{complex}, we mean those which require all variables in the BN for inference. For example, for a size 2 BN with variables $A$ and $B$ where $A$ is the parent of $B$, all possible queries are $P(A), P(B), P(A, B), P(A|B)$, and $P(B|A)$. Among these, $P(A)$ is the only query that is not \textit{complex} since it can be answered only with the CPT of $A$ and, therefore, is not selected. We assign true/false values randomly to the query variables.

\subsection{Textual Context and Question}

After generating the BNs, CPTs, and queries, we create the textual context and question that describes the BN (mapping every entry in the CPTs to natural language) and the query, respectively. For the context, sentences follow two templates: 1) For explaining prior probabilities, we use the template \textit{"\{node name\} is True/False with \#\#\% Probability"}. 2) For explaining dependent nodes, we use the template \textit{"\{node name\} is True/False with \#\#\% Probability, if \{parent node name\} is True/False"}. Figure~\ref{fig:dataset_example} shows the context for a given BN and CPTs. The template for textual questions is: \textit{"What is the probability that \{a node name\} is True/False and \ldots given that \{a node name\} is True/False and \ldots ?"}. For a query without evidence variables, the text after ``given'' is omitted. For example, $P(A|\sim\!B)$ would be translated to the textual query (question), "What is the probability that A event is True given that B event is False?".

\subsection{Verification and Inference}

At the final step of our dataset generation, we use the Python library pgmpy~\cite{ankan2015pgmpy} to verify the soundness of our BN, probabilities, and queries and to infer the answer to the queries. This library, which is specifically designed for creating and working with Bayesian Networks, takes our generated CPTs as input, verifies their soundness and answers the queries via an exact inference method.

\section{Methodology}
Here, we introduce our approach to probabilistic reasoning with LLMs. We use a basic QA and a Chain of Thought prompting as baselines and propose new strategies for mapping to symbolic representations.

\subsection{Baselines}

\subsubsection{Basic QA Prompting} In this prompting approach,  we ask the model to generate a single numerical answer to the probabilistic question. We experiment with zero-shot and few-shot settings. In the zero-shot setting, only the instruction, context, and question are in the prompt. Figure~\ref{fig:threexampledemostration} shows the few-shot setting with an in-context example included.

\subsubsection{Chain of Thought (COT)} Following COT~\cite{NEURIPS2022_9d560961}, we prompt the LLM to explain its reasoning process while refraining completely (in zero-shot setting) or partially (in few-shot setting) from imposing a strict solution structure. COT's prompting structure is similar to Basic QA's, except that the requested answer should explain the mathematical reasoning, calculate the final answer, and generate the target output probability at the end, as shown in Figure~\ref{fig:threexampledemostration}.

\begin{figure*}
    \centering
    \includegraphics[width=0.98\linewidth]{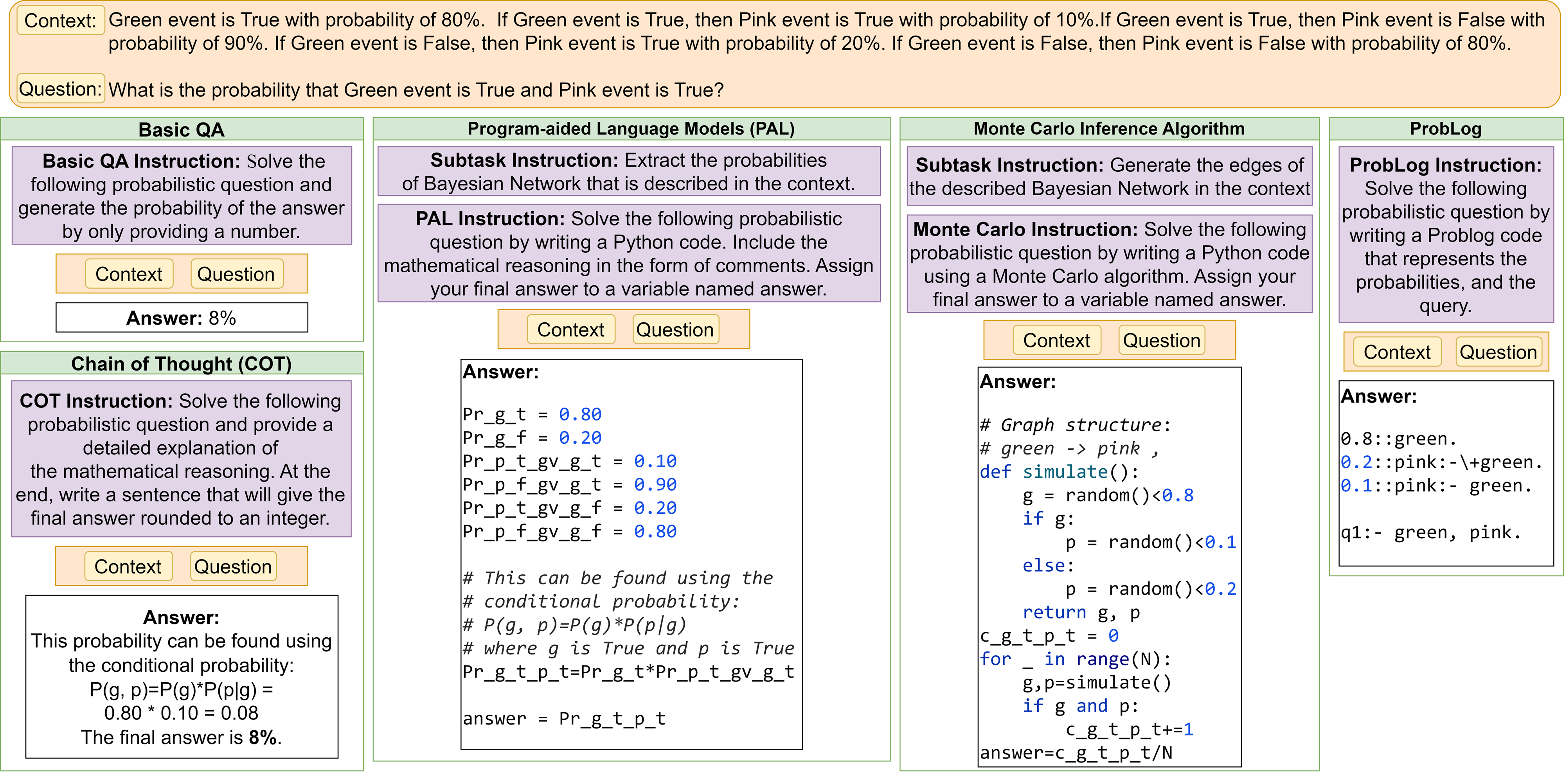}
\caption{This figure shows our main prompting approaches, PAL, Monte Carlo, and ProbLog, alongside the baseline approaches, Basic QA and COT. Each prompt begins with an instruction (purple boxes) that describes the problem and the answer format. Then, the context, question, and answer are demonstrated depending on the approach. We display only the first in-context example here but use 3 in our experiments. When we require the use of our designed subtasks in the prompt, their instructions and answers are prepended to the main approach, as shown in the PAL method for \textit{Number Extraction} and the Monte Carlo method for \textit{Graph Generation}.}

\label{fig:threexampledemostration}
\end{figure*}

\subsection{Structured Prompting with Subtasks}
\label{sec:ne-gg}

%Given the complexity of the probabilistic inference, we expect that performing it with additional subtasks would benefit the LLM. 

Given the complexity of the probabilistic inference, we propose to divide the problem into multiple steps and demonstrate the steps to the LLMs in one prompt. This approach has shown to be effective in other similar research~\cite{jin2023CLADDER,poesia2023certified}. The most intuitive step for our problem is identifying the probabilities of the single events and the conditional probabilities. Another important step is to recognize the variables' dependencies, which are the edges in the corresponding BN. Consequently, we use the extraction of probability values from text, named \textit{Number Extraction} subtask, and probabilistic dependencies \textit{Graph Generation} subtask as the prior steps to final reasoning.

\subsubsection{Number Extraction} In this subtask, the LLM should extract the CPT probabilities from the input context and output them in a structured format as Python-compatible variable assignments. Each line of the output represents either a probability of an event or a conditional probability. This format is shown in the ``PAL'' column of Figure~\ref{fig:threexampledemostration}. To add this subtask, we prepend its corresponding instruction to the prompt, i.e., ``Extract the probabilities...'' and its answers to the in-context examples. This subtask aims to facilitate the correct extraction of probabilities before reasoning to avoid hallucination~\cite{ouyang2022training} of incorrect numbers. 

\subsubsection{Graph Generation} In this subtask, we want the LLM to generate the underlying Bayesian Network of the given context as a list of edges, each indicating the direct dependency between two variables. For a BN with $V$ variables, the output should consist of $V-1$ edges in the format of $v_i -> v_j$ separated by `,'. This will help the model capture the random variables' dependency structure and utilize it in mapping to symbolic solutions. Similar to \textit{Number Extraction}, to use \textit{Graph Generation}, its instruction and answers are added to the prompts. An example of a simple Bayesian Network with two nodes, Green and Pink, is shown in the ``Monte Carlo Inference Algorithm'' column of Figure~\ref{fig:threexampledemostration}.

% An example of this format is demonstrated in the column that explains the \textbf{MC} prompting in Figure~\ref{fig:threexampledemostration}.
 
\subsection{Mapping to Symbolic Computations}

\subsubsection{Program-aided Language Models (PAL)} PAL~\cite{gao2023pal} is the first study to analyze the use of Python in various mathematical reasoning QA datasets. However, most problems tested in PAL require only a few lines of code, unlike BLInD, which may need complex, multi-line calculations depending on the BN. A benefit of the PAL method is it bypasses the challenge of mathematical calculations by the LLM itself. Similar to the original PAL paper, we instruct the LLM to solve the problem by explaining the mathematical reasoning process and mapping to the basic arithmetic calculations in Python code leading to the answer. Figure~\ref{fig:threexampledemostration} shows an in-context example of this approach, where the \textit{Number Extraction} subtask is first used, followed by mapping the reasoning solution to Python code.

\subsubsection{Monte Carlo Inference Algorithm} Given the efficiency and popularity of Monte Carlo Algorithms~\cite{KollerFriedman2009} for approximate inference, we try to map our problem the \textit{Direct Sampling} technique for inference. An LLM can use this method by generating a Python function, we call \textit{simulate}, that samples all events according to the probabilistic dependencies expressed in the BN. Here, all parent variables must be sampled before their children. Keeping this order is the main challenge in mapping to this algorithm with LLMs. Figure~\ref{fig:threexampledemostration} shows an in-context example of the Monte Carlo method with the \textit{simulate} function defined as a part of the answer. The LLM is also instructed to call this function in the generated Python code which leads to computing the answer to the probabilistic question. 

%In the ``Answer'' part of the prompt of this method in Figure~\ref{fig:threexampledemostration}, the \textit{Graph Generation} subtask instruction asks the LLM to generate the Bayesian graph prior to the generation of Monte Carlo code.

\subsubsection{Probabilistic Logical Solver (ProbLog)} In our neuro-symbolic method, we employ a technique that involves mapping the context and question to a probabilistic logical formalism. We use ProbLog~\cite{10.5555/1625275.1625673}, a probabilistic programming language that extends Prolog~\cite{Bratko00} to incorporate probabilistic logical reasoning. Here, the LLM is asked to generate a ProbLog code corresponding to the probabilities given in the context and to create the formal \textit{ProbLog query} based on the question. We subsequently execute the ProbLog code and extract the final answer. An in-context example of this code is shown in Figure~\ref{fig:threexampledemostration}. The LLM does not need extensive ProbLog programming knowledge; The three in-context examples we supply are sufficiently complex and encompass all the necessary ProbLog syntax information to enable the generation of code for our contexts and questions.

\begin{table*}
\begin{center}
\small
\begin{tabular}{|c|c|P{0.4cm}|P{0.4cm}|P{0.4cm}|P{0.4cm}|P{0.4cm}|P{0.4cm}|P{0.4cm}|P{0.4cm}|P{0.4cm}||P{0.95cm}|P{0.95cm}|P{0.95cm}|}
\hline
Model&Method & $V_2$ & $V_3$ & $V_4$ & $V_5$ & $V_6$ & $V_7$ & $V_8$ & $V_9$ & $V_{10}$ & $V_{2-5}$ & $V_{6-10}$ & $V_{2-10}$ \\

\hline
\multirow{ 4}{*}{GPT3.5} & Basic QA ZS  &\HeatCell{33}&\HeatCell{13}&\HeatCell{5}&\HeatCell{4}&\HeatCell{6}&\HeatCell{2}&\HeatCell{3}&\HeatCell{1}&\HeatCell{2}&\HeatCell{13}&\HeatCell{2}& \HeatCell{7}  \\
&Basic QA FS  &\HeatCell{3}&\HeatCell{0}&\HeatCell{1}&\HeatCell{1}&\HeatCell{2}&\HeatCell{2}&\HeatCell{1}&\HeatCell{1}&\HeatCell{0}&\HeatCell{1}&\HeatCell{1}& \HeatCell{1}  \\
&COT ZS  &\HeatCell{53}&\HeatCell{8}&\HeatCell{4}&\HeatCell{5}&\HeatCell{10}&\HeatCell{5}&\HeatCell{2}&\HeatCell{2}&\HeatCell{0}&\HeatCell{17}&\HeatCell{3}& \HeatCell{9}  \\
&COT FS  &\HeatCell{52}&\HeatCell{23}&\HeatCell{12}&\HeatCell{5}&\HeatCell{8}&\HeatCell{4}&\HeatCell{1}&\HeatCell{4}&\HeatCell{2}&\HeatCell{23}&\HeatCell{3}& \HeatCell{12}  \\

\hline
\multirow{ 4}{*}{Llama3} & Basic QA ZS &\HeatCell{31}&\HeatCell{21}&\HeatCell{5}&\HeatCell{6}&\HeatCell{6}&\HeatCell{5}&\HeatCell{1}&\HeatCell{1}&\HeatCell{0}&\HeatCell{16}&\HeatCell{3}& \HeatCell{8}  \\
&Basic QA FS &\HeatCell{3}&\HeatCell{0}&\HeatCell{1}&\HeatCell{1}&\HeatCell{2}&\HeatCell{2}&\HeatCell{1}&\HeatCell{1}&\HeatCell{0}&\HeatCell{1}&\HeatCell{1}& \HeatCell{1}  \\
&COT ZS &\HeatCell{63}&\HeatCell{45}&\HeatCell{21}&\HeatCell{17}&\HeatCell{18}&\HeatCell{11}&\HeatCell{9}&\HeatCell{4}&\HeatCell{2}&\HeatCell{37}&\HeatCell{9}& \HeatCell{21}  \\
&COT FS &\HeatCell{63}&\HeatCell{46}&\HeatCell{21}&\HeatCell{12}&\HeatCell{20}&\HeatCell{15}&\HeatCell{7}&\HeatCell{8}&\HeatCell{5}&\HeatCell{36}&\HeatCell{11}& \HeatCell{22}  \\

\hline
\multirow{ 4}{*}{GPT4} &Basic QA ZS &\HeatCell{44}&\HeatCell{23}&\HeatCell{9}&\HeatCell{9}&\HeatCell{11}&\HeatCell{11}&\HeatCell{8}&\HeatCell{8}&\HeatCell{2}&\HeatCell{21}&\HeatCell{8}& \HeatCell{14}  \\
&Basic QA FS &\HeatCell{3}&\HeatCell{0}&\HeatCell{1}&\HeatCell{1}&\HeatCell{2}&\HeatCell{2}&\HeatCell{1}&\HeatCell{1}&\HeatCell{0}&\HeatCell{1}&\HeatCell{1}& \HeatCell{1}  \\
&COT ZS &\HeatCell{79}&\HeatCell{63}&\HeatCell{27}&\HeatCell{10}&\HeatCell{17}&\HeatCell{6}&\HeatCell{5}&\HeatCell{7}&\HeatCell{6}&\HeatCell{45}&\HeatCell{8}& \HeatCell{24}  \\
&COT FS &\HeatCell{78}&\HeatCell{64}&\HeatCell{36}&\HeatCell{25}&\HeatCell{22}&\HeatCell{16}&\HeatCell{7}&\HeatCell{7}&\HeatCell{7}&\HeatCell{50}&\HeatCell{12}& \HeatCell{29}  \\

\hline
\end{tabular}
\end{center}
\caption{Comparison of GPT3.5, Llama3, and GPT4 accuracy results for Basic QA and COT methods, presented as percentages. The columns represent dataset splits $V_i$, and the average results for smaller BNs $V_{2-5}$, larger BNs $V_{6-10}$, and all BNs $V_{2-10}$. The rows show the methods tested with zero-shot (ZS) or few-shot (FS) settings.}
\label{tab:DirectandCOTGPT3GPT4Results}
\end{table*}

\section{Experiments}

Here, we present the results of our experiments on our baselines and proposed prompting techniques, which we evaluate on BLInD and an adapted version of the CLADDER dataset. 

\noindent\textbf{LLM Models} In our experiments, we employ three LLMs: Llama3~\cite{meta2024llama3}, specifically the \textit{meta-llama-3-70b-instruct} variant; GPT3.5~\cite{brown2020language}, using the \textit{gpt-3.5-turbo-0613} release; and GPT4~\cite{openai2023GPT4}, with the \textit{gpt-4-0613} version. These models are evaluated in zero-shot and few-shot settings without any fine-tuning. See Appendix B for details about hyper-parameters and Appendix C for the results of the other models that we tested, such as Mistral~\cite{jiang2023mistral} and Llama2~\cite{touvron2023llama}.

\noindent\textbf{Few-shot Example Selection} We selected a set of shots from a development dataset and manually crafted their solutions. After evaluating these shots on the same development dataset, we identified the three most effective examples through an iterative, trial-and-error process. To ensure a fair comparison, we consistently use these three examples across all models and methods rather than tailoring the examples to each specific approach or model. See Appendix E for more details about the prompts, examples, and solutions.

\noindent\textbf{Evaluation Metrics.} Given a context and a question, we consider an answer probability to be correct if it is within the $\pm 0.01$ range of the ground truth probability (ex. any answer within $[0.30-0.32]$ is correct for a ground truth of $0.31$). We chose this threshold because we found that the \textbf{outputs were either correct or wrong with a large margin} since the exact line of computations is not followed in those cases. This bimodal behavior, which differs from traditional regression models, renders evaluation metrics such as MSE and L1 ineffective. In this context, correct predictions contribute minimally to the error, while incorrect predictions dominate the error metric in a way that lacks relevance. Additionally, this behavior rendered larger thresholds useless, as the accuracy at a threshold of $0.01$ was nearly identical to that at $0.05$. A narrower threshold would cause the challenge of number precision which is not in our interest due to the nature of our task and has been previously highlighted in~\cite{gao2023pal} for other mathematical reasoning problems. For the evaluation of the subtasks, we count an output as accurate if \textbf{all} the numbers in \textit{Number Extraction} and all the edges in \textit{Graph Generation} are correctly generated without redundancy. As a result, the numbers in all Tables are accuracy values in percentages based on these criteria.

\noindent\textbf{Evaluation Splits of the Dataset.} 
To assess our methods, we randomly select 100 instances from each data split $V_i$, resulting in a total of 900 instances. This test set remains consistent across all of our LLMs.

\subsection{Solving Probabilistic Questions Directly}
\label{sec:baselinemethodsresults}
Here, we apply the baseline methods of Basic QA and COT, focusing on answering probabilistic questions directly. Their performance is detailed in Table~\ref{tab:DirectandCOTGPT3GPT4Results}. In Basic QA, overall, the results are very low, and only GPT4 achieves meaningful results for some of the dataset splits with \textbf{smaller BNs}, i.e. $V_i$ with $i<=5$. In the few-shot setting of Basic QA, the additional examples, which do not explain their solutions, worsen the results for all the LLMs. Using COT improved the results for all models. However, even with COT, these models struggle particularly in dataset splits with \textbf{larger BNs}, i.e. $V_i$ with $i>5$. We will use these baselines to compare with our symbolic methods.

\begin{table}
\small
\centering
\begin{tabular}{|P{1.25cm}|P{0.31cm}|p{0.31cm}|P{0.31cm}|P{0.31cm}|P{0.31cm}|P{0.31cm}|P{0.31cm}|P{0.31cm}|P{0.31cm}|}
\hline
 LLM / $V_i$ & $V_2$ & $V_3$ & $V_4$ & $V_5$ & $V_6$ & $V_7$ & $V_8$ & $V_9$ & $V_{10}$ \\
\hline
GPT3.5 & 100 & 100 & 100 & 100 & 96 & 95 & 98 & 94 & 94 \\
\hline
Llama3 & 100 & 100 & 100 & 100 & 100 & 100 & 99 & 100 & 100 \\
\hline
GPT4 & 100 & 100 & 100 & 100 & 100 & 100 & 100 & 100 & 100 \\
\hline
\end{tabular}
\caption{\textit{Number Extraction} accuracy of our models, presented as percentages and based on the exact match of all the extracted probabilities of the context.}
\label{table:number_extraction}
\end{table}

\subsection{Subtasks}
\label{sec:NEGGresults}

Before reporting the results of the final symbolic solvers, we discuss the results of \textit{Number Extraction} and \textit{Graph Generation}. The results of \textit{Number Extraction} are shown in Table~\ref{table:number_extraction}, which indicates this subtask is quite straightforward to solve. Llama3 and GPT4 extract all numbers correctly, achieving 100\% accuracy in all $V_i$s. For GPT3.5, although the accuracy drops as the number of variables increases, it remains overall very high above 90\%. 

\begin{table}
\small
\centering
\begin{tabular}{|P{1.25cm}|P{0.31cm}|p{0.31cm}|P{0.31cm}|P{0.31cm}|P{0.31cm}|P{0.31cm}|P{0.31cm}|P{0.31cm}|P{0.31cm}|}
\hline
LLM / $V_i$ & $V_2$ & $V_3$ & $V_4$ & $V_5$ & $V_6$ & $V_7$ & $V_8$ & $V_9$ & $V_{10}$ \\
\hline					
GPT3.5 & 100 & 95 & 92 & 93 & 84 & 75 & 79 & 73 & 78 \\
\hline
Llama3 & 99 & 99 & 99 & 100 & 99 & 95 & 94 & 93 & 89 \\
\hline
GPT4 & 100 & 100 & 100 & 100 & 100 & 100 & 100 & 100 & 100 \\
\hline
\end{tabular}
\caption{\textit{Graph Generation} accuracy, presented as percentages. The extracted graph should exactly match the correct BN graph to be counted as correct.}
\label{table:graph_generation}
\end{table}

The results of the more challenging \textit{Graph Generation} subtask are shown in Table~\ref{table:graph_generation}. Mirroring the pattern observed in Table~\ref{table:number_extraction}, we notice a decline in accuracy as the number of variables increases. However, the drop in accuracy is more notable and goes from 100\% in $V_2$ to as low as 73\% in $V_9$ for GPT3.5. GPT4 generated all graphs correctly in all $V_i$s. Looking at the results in Table~\ref{table:graph_generation}, we observe minor inconsistencies such as a lower accuracy for $V_9$ compared to $V_{10}$. These inconsistencies stem from the inherent randomness in the output generation of LLMs and our random selection dataset instances. These small inconsistencies happen in some other parts of our experiments but they do not detract from the core message and pattern of our findings.

Note that the accuracies reported here are calculated when each subtask is prompted to the LLM as a standalone problem. When integrating these subtasks within our solutions, we prompt the LLM to generate both the subtask and the problem solution together. This affects the subtasks' accuracy and, consequently, usefulness depending on the symbolic method, as discussed in the next section.

\begin{table*}
\begin{center}
\small
\begin{tabular}{|c|c|P{0.4cm}|P{0.4cm}|P{0.4cm}|P{0.4cm}|P{0.4cm}|P{0.4cm}|P{0.4cm}|P{0.4cm}|P{0.4cm}||P{0.95cm}|P{0.95cm}|P{0.95cm}|}
\hline
Model&Method & $V_2$ & $V_3$ & $V_4$ & $V_5$ & $V_6$ & $V_7$ & $V_8$ & $V_9$ & $V_{10}$ & $V_{2-5}$ & $V_{6-10}$ & $V_{2-10}$ \\

\hline

\multirow{ 5}{*}{GPT3.5} &PAL &\HeatCell{66}&\HeatCell{34}&\HeatCell{25}&\HeatCell{17}&\HeatCell{14}&\HeatCell{9}&\HeatCell{6}&\HeatCell{5}&\HeatCell{2}&\HeatCell{35}&\HeatCell{7}& \HeatCell{19}  \\
&PAL w/NE  &\HeatCell{85}&\HeatCell{66}&\HeatCell{41}&\HeatCell{27}&\HeatCell{19}&\HeatCell{12}&\HeatCell{5}&\HeatCell{3}&\HeatCell{6}&\HeatCell{54}&\HeatCell{9}& \HeatCell{29}  \\
&Monte Carlo  &\HeatCell{79}&\HeatCell{63}&\HeatCell{71}&\HeatCell{65}&\HeatCell{41}&\HeatCell{32}&\HeatCell{33}&\HeatCell{18}&\HeatCell{14}&\HeatCell{69}&\HeatCell{27}& \HeatCell{46}  \\
&Monte Carlo w/GG  &\HeatCell{85}&\HeatCell{82}&\HeatCell{83}&\HeatCell{68}&\HeatCell{42}&\HeatCell{31}&\HeatCell{28}&\HeatCell{18}&\HeatCell{8}&\HeatCell{79}&\HeatCell{25}& \HeatCell{49}  \\
&ProbLog  &\HeatCell{87}&\HeatCell{82}&\HeatCell{88}&\HeatCell{75}&\HeatCell{59}&\HeatCell{52}&\HeatCell{46}&\HeatCell{38}&\HeatCell{35}&\HeatCell{83}&\HeatCell{46}& \HeatCell{62}  \\

\hline
\multirow{ 5}{*}{Llama3} &PAL  &\HeatCell{100}&\HeatCell{84}&\HeatCell{57}&\HeatCell{36}&\HeatCell{31}&\HeatCell{20}&\HeatCell{10}&\HeatCell{14}&\HeatCell{8}&\HeatCell{69}&\HeatCell{17}& \HeatCell{40}  \\
&PAL w/NE  &\HeatCell{100}&\HeatCell{95}&\HeatCell{71}&\HeatCell{52}&\HeatCell{46}&\HeatCell{28}&\HeatCell{16}&\HeatCell{16}&\HeatCell{9}&\HeatCell{79}&\HeatCell{23}& \HeatCell{48}  \\
&Monte Carlo  &\HeatCell{100}&\HeatCell{100}&\HeatCell{96}&\HeatCell{96}&\HeatCell{92}&\HeatCell{85}&\HeatCell{77}&\HeatCell{72}&\HeatCell{64}&\HeatCell{98}&\HeatCell{78}& \HeatCell{87}  \\
&ProbLog  &\HeatCell{90}&\HeatCell{95}&\HeatCell{92}&\HeatCell{87}&\HeatCell{95}&\HeatCell{94}&\HeatCell{87}&\HeatCell{82}&\HeatCell{78}&\HeatCell{91}&\HeatCell{87}& \HeatCell{89}  \\

\hline
\multirow{ 4}{*}{GPT4} &PAL&\HeatCell{100}&\HeatCell{86}&\HeatCell{70}&\HeatCell{58}&\HeatCell{50}&\HeatCell{27}&\HeatCell{21}&\HeatCell{14}&\HeatCell{7}&\HeatCell{78}&\HeatCell{24}& \HeatCell{48}  \\
&PAL w/NE  &\HeatCell{99}&\HeatCell{96}&\HeatCell{78}&\HeatCell{64}&\HeatCell{43}&\HeatCell{26}&\HeatCell{14}&\HeatCell{14}&\HeatCell{10}&\HeatCell{84}&\HeatCell{21}& \HeatCell{49}  \\
&Monte Carlo &\HeatCell{100}&\HeatCell{99}&\HeatCell{98}&\HeatCell{100}&\HeatCell{92}&\HeatCell{94}&\HeatCell{92}&\HeatCell{90}&\HeatCell{88}&\HeatCell{99}&\HeatCell{91}& \HeatCell{94}  \\
&Monte Carlo w/GG &\HeatCell{100}&\HeatCell{97}&\HeatCell{99}&\HeatCell{98}&\HeatCell{97}&\HeatCell{96}&\HeatCell{88}&\HeatCell{92}&\HeatCell{85}&\HeatCell{99}&\HeatCell{92}& \HeatCell{95}  \\
&ProbLog &\HeatCell{99}&\HeatCell{98}&\HeatCell{100}&\HeatCell{100}&\HeatCell{96}&\HeatCell{97}&\HeatCell{97}&\HeatCell{98}&\HeatCell{96}&\HeatCell{99}&\HeatCell{97}& \HeatCell{98}  \\

\hline
\end{tabular}
\end{center}
\caption{GPT3.5, Llama3, and GPT4 accuracy results, presented as percentages, for the PAL, Monte Carlo, and ProbLog methods. w/NE and w/GG denote the inclusion of \textit{Number Extraction} and \textit{Graph Generation}. The columns represent dataset splits $V_i$, and the average results for smaller BNs $V_{2-5}$, larger BNs $V_{6-10}$, and all BNs $V_{2-10}$.}
\label{tab:methodsGPT3GPT4Results}
\end{table*}

\subsection{Evaluation of Proposed Methods}
\label{sec:mainresults}

% can be shortened
Here, we assess our three proposed approaches,  PAL, Monte Carlo, and ProbLog combined with \textit{Number Extraction} and \textit{Graph Generation} in three LLMs. We discuss how effective the subtasks are with each method and analyze their impact based on factors like the number of variables and the employed LLM. The results of these experiments are presented in Table~\ref{tab:methodsGPT3GPT4Results}. Not all combinations lead to better performance; some reduce overall accuracy. As such, these underperforming configurations have been excluded from the main table. For a comprehensive overview, including the underperforming configurations, refer to Appendix C.

\subsubsection{PAL, Monte Carlo, and ProbLog} As seen in Table~\ref{tab:methodsGPT3GPT4Results}, there is a significant improvement in the performance of all of these methods, compared to Basic QA and COT (previously shown in Table~\ref{tab:DirectandCOTGPT3GPT4Results}). Additionally, accuracy consistently increases across all models (closed and open-source) by transitioning from PAL to Monte Carlo and then to ProbLog. This suggests that the proposed methods' effectiveness is independent of the LLMs. All models struggle to generate a solution with PAL for larger BNs. In contrast, when we utilize the Monte Carlo approach, the accuracy of these larger BNs sharply increases, suggesting proficiency of LLMs at mapping the entire BN correctly to a Monte Carlo algorithm code, even for a large number of variables. 

ProbLog eliminates the challenge of structural programming and requires only the correct extraction of probabilities (represented declaratively) and generating a corresponding \textit{ProbLog query}. In this case, GPT4 can solve almost every question. GPT3.5 is mainly held back by the challenge of writing probabilistic logical programming code. While Llama3 (like GPT4) featured nearly 100\% correct Python syntax in the PAL and Monte Carlo methods, it sometimes fails to create coherent ProbLog code. This leads to somewhat inconsistent performance among smaller BNs. See the Appendix D for details of coding syntax errors.

\subsubsection{PAL with Number Extraction} This combination, which is shown as ``PAL w/NE'' in Table~\ref{tab:methodsGPT3GPT4Results}, shows that the accuracy of both GPT3.5 and Llama3 benefits from the addition of the \textit{Number Extraction} subtas to the PAL prompt. This appears to reduce the hallucination~\cite{ouyang2022training} of probability values in PAL solutions, as we further confirmed by analyzing several test cases. This subtask was not needed for the more robust GPT4, which can remember the numbers and its addition resulted in marginal improvements.

\subsubsection{Monte Carlo with Graph Generation} The accuracy of GPT3.5 and GPT4 improved when the Monte Carlo method was combined with the \textit{Graph Generation} subtask, shown as `Monte Carlo w/GG'' in Table~\ref{tab:methodsGPT3GPT4Results}. \textit{Graph Generation} subtask further improves the already high performance of Monte Carlo for larger BNs and enables GPT4 to reach a near-perfect average accuracy of 95\% in this setting. The improvements caused by the addition of \textit{Graph Generation} are not surprising since the Monte Carlo method generates a Python function with many nested ``if'' structures tied to the underlying Bayesian Network's graph structure. However, Llama3 is the exception and the only model that does not benefit from this configuration, as discussed further below.

\subsection{Discussion}

\subsubsection{Practical Use of Subtasks} While our proposed approaches proved effective, independently of the LLMs, this was not true for our subtasks. As mentioned earlier, we comprehensively tested all the configurations, but not all of them improved our models. This raises a few questions: Q1) Why do the added subtasks not always help? For example, in principle, the Monte Carlo method could use \textit{Number Extraction} in its code to improve. Q2) Why does ProbLog not improve with any subtasks? Q3) Why does Llama3 not improve with \textit{Graph Generation} in its Monte Carlo method like GPT3.5 and GPT4? Q4) Why does no method improve by adding both subtasks together? To answer these questions, we looked at subtask generation and their utilization by LLMs more closely, which led to two main findings.

The first finding is that in contrast to most mathematical problems tested with LLMs, which require brief solutions~\cite{kim-etal-2023-aint,frieder2023mathematical}, our dataset demands the generation of large outputs. When the added information by subtasks is not exploited effectively and lengthens the output even more for no reason, it leads to a notable drop in the LLM performance. For example, while the graph structure is intuitively helpful for probabilistic inference, the Python code in PAL does not utilize it directly. This directly addresses Q1 and touches on Q4. The main bottleneck of the ProbLog method was the syntax errors that subtasks could help with, which answers Q2.

The second finding concerns the accuracy of the subtasks, which drops when generated in the same prompt with the main solution (as we prompt the LLM only once). This puts Llama3 in a precarious position regarding the Monte Carlo method with its high accuracy. For \textit{Graph Generation} to further improve this method, it has to have an accuracy higher than the method to be helpful. However, that is not the case for this configuration for Llama3. For instance, the accuracy of Llama3 in the Monte Carlo method for $V_{10}$ is 64\% (Table~\ref{tab:methodsGPT3GPT4Results}), which is already higher than \textit{Graph Generation} accuracy for $V_{10}$ that is 56\% when generated in the same prompt (See the Supplementary for detailed results). For GPT3.5 and GPT4, \textit{Graph Generation} accuracy remains high enough, which in the case of GPT3.5 is partially due to its weaker performance in the Monte Carlo method. This finding resolves Q3 and provides further insights into Q4. 

\subsubsection{Trade-off Between Complexity and Effectiveness}
Among our 5 methods, the Simplest and the most efficient one is the Basic QA, which generates a few tokens. COT slightly improves the results at the cost of more tokens in the input and output. There is a significant improvement in accuracy, moving to our main methods with external tools. Their LLM code generation time is the same as COT, but they need the additional time to execute the generated program. According to our experiments, the probabilistic inference run-time is negligible compared to the inference run-time of LLM output generation which is a few milliseconds versus seconds taken by LLMs. However, from an algorithmic perspective for probabilistic inference, ProbLog will be more complex compared to Monte Carlo sampling as it needs to deal with logical representations. Given that a probabilistic network question described in a natural language context forms rather small Bayesian Networks, our methods are practical for solving this problem.

\subsubsection{Use of External Tools} Using external tools with LLMs is an area of research that leverages the LLMs and exploits diverse computational paradigms~\cite{schick2023toolformerlanguagemodelsteach}. The tools we use are 1) pure Python for PAL and Monte Carlo methods and 2) Python plus the underlying ProbLog engine. All these tools are open-source, and conversion to their Python interface is highly accurate using language models. They are more efficient compared to LLMs, as discussed above, and their operation as black-box executables requires minimal computing resources.

\subsection{Adaptation of the CLADDER Dataset}
\label{sec:cladderresults}
\begin{figure}
    \centering
    \includegraphics[width=0.98\linewidth]{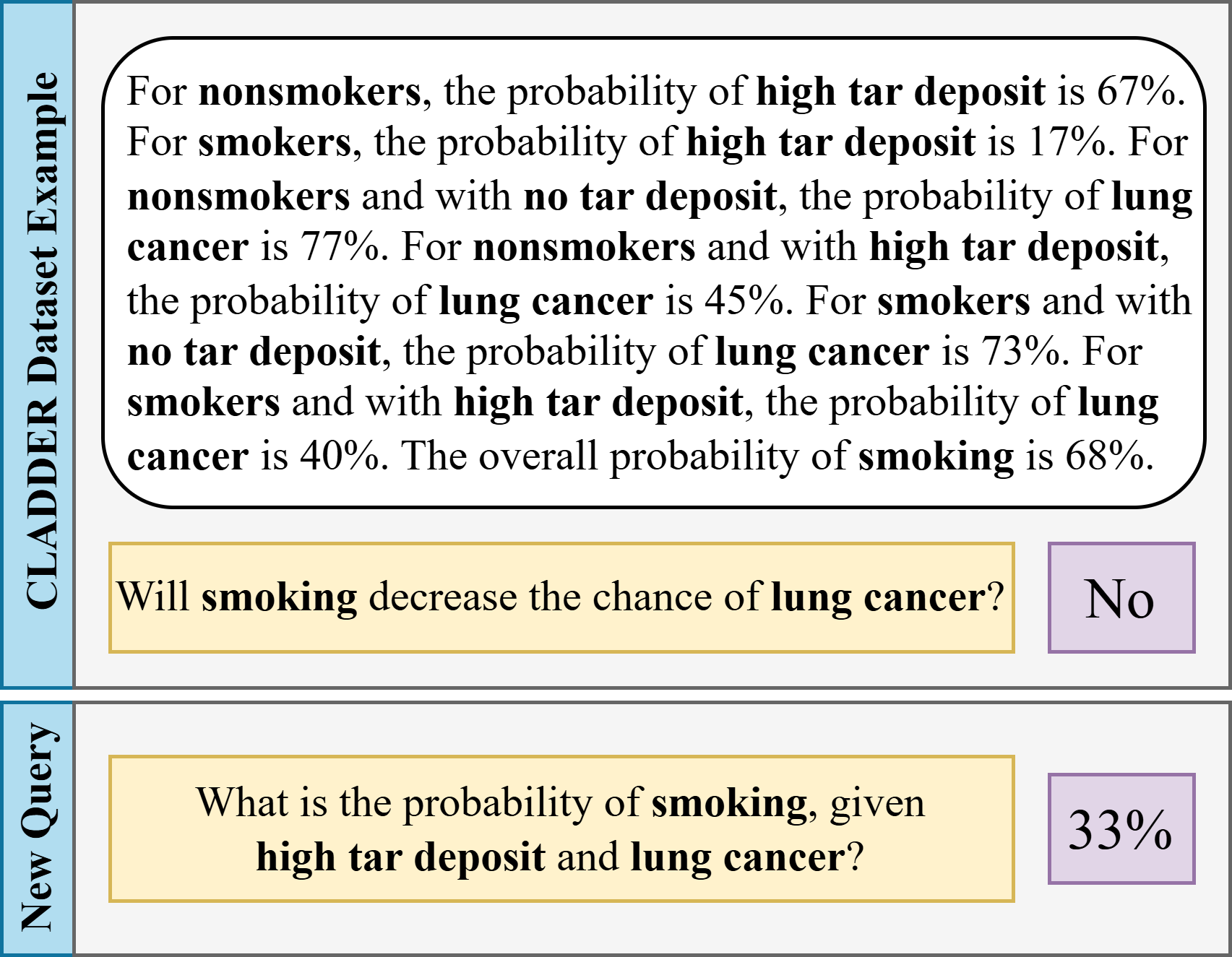}
    \caption{An example from the hard test subset of CLADDER dataset and a corresponding generated probabilistic query. The top section displays the context with events in bold font (white box), a query (yellow box), and the binary (Yes/No) answer (purple box). The bottom section presents an example of a probabilistic query derived from the same context, which requires a probability-based response.}
    \label{fig:cladderexample_new}
\end{figure}

We conclude our experiments by testing our methods on an adaptation of the CLADDER dataset~\cite{jin2023CLADDER}. This dataset is designed to test the causal reasoning capabilities of LLMs. The contexts of this QA dataset describe a probabilistic causal structure with a maximum of 4 variables, designed from natural-sounding templates. Questions in this dataset mostly require a binary yes/no answer and not a probability. Our results are, thus, not comparable to the ones in~\cite{jin2023CLADDER}. Using the natural-sounding contexts in CLADDER's \textit{hard tests} split, we created challenging queries for the contexts and sample 100 instances. Figure~\ref{fig:cladderexample_new} provides an example of this dataset along with one of our generated queries for its context.

The results of our tests on this adaptation of CLADDER are shown in Table~\ref{tab:CLADDER}, which follow the same trend seen in BLInD for smaller BNs. For example, adding \textit{Graph Generation} to the Monte Carlo method does not improve the model here as it was most helpful when the number of variables was large. This consistency with BLInD evaluations further solidifies our claims. When testing our methods on the CLADDER, we used the same in-context examples of BLInD without tailoring them to the more natural contexts of CLADDER as we found it unnecessary. The performance remained very high, as seen in Table~\ref{tab:CLADDER}. Based on the results from BLInD and CLADDER, our experiments suggest that the difficulty of probabilistic reasoning over text is not directly correlated with the naturalness and sophistication of the language. Instead, it depends on the depth of reasoning required and the number of variables involved.

\begin{table}
\centering
\small
\centering
\begin{tabular}{|c|c|c|c|}

\hline
Method  & GPT3.5 & Llama3 & GPT4 \\
\hline
Basic QA ZS & \HeatCell{0} & \HeatCell{0} & \HeatCell{20} \\
Basic QA FS & \HeatCell{0} & \HeatCell{0} & \HeatCell{0} \\
\hline
COT ZS & \HeatCell{9} & \HeatCell{47} & \HeatCell{65} \\
COT FS & \HeatCell{3} & \HeatCell{38} & \HeatCell{64} \\
\hline
PAL & \HeatCell{26} & \HeatCell{91} & \HeatCell{96}\\
PAL w/NE & \HeatCell{39} & \HeatCell{96} & \HeatCell{96}\\
\hline
Monte Carlo & \HeatCell{75} & \HeatCell{96} & \HeatCell{98} \\
Monte Carlo w/GG & \HeatCell{75} & \HeatCell{95} & \HeatCell{97} \\
\hline
ProbLog & \HeatCell{71} & \HeatCell{84} & \HeatCell{97}\\
\hline
\end{tabular}

\caption{Accuracy results of the CLADDER dataset as percentages. w/NE, w/GG, ZS and FS denote use of \textit{Number Extraction}, \textit{Graph Generation}, zero-shot and few-shot.}
\label{tab:CLADDER}
\end{table}

\section{Conclusion and Future Work}
In this work, we introduced BLInD, a new dataset for dealing with uncertain text and evaluating the capabilities of LLMs on probabilistic reasoning over text with explicitly quantified uncertainty. We proposed several prompt engineering techniques, mapping the problem to different formal representations, including Python low-level arithmetic computations, approximate inference algorithms, and probabilistic logical programming. Our evaluations demonstrated that our main methods significantly improve the performance of LLMs on BLInD and on an adapted version of another dataset, CLADDER, with natural-sounding contexts.

Our methods solve probabilistic questions without any fine-tuning or modification to the architecture of LLMs. As an interesting direction to continue this work, future research could explore alterations to open-source LLMs' architectures and training objectives specifically designed for probabilistic inference.

\section*{Acknowledgment}

This project is supported by the Office of Naval Research (ONR) grant N00014-23-1-2417. Any opinions, findings, and conclusions or recommendations expressed in this material are those of the authors and do not necessarily reflect the views of Office of Naval Research.

\bibliography{aaai25,anthology}

\begin{thebibliography}{38}
\providecommand{\natexlab}[1]{#1}

\bibitem[{AI@Meta(2024)}]{meta2024llama3}
AI@Meta. 2024.
\newblock Llama 3 Model Card.

\bibitem[{Ankan and Panda(2015)}]{ankan2015pgmpy}
Ankan, A.; and Panda, A. 2015.
\newblock pgmpy: Probabilistic graphical models using python.
\newblock In \emph{Proceedings of the 14th Python in Science Conference (SCIPY
  2015)}. Citeseer.

\bibitem[{Bratko(2000)}]{Bratko00}
Bratko, I. 2000.
\newblock \emph{Prolog Programming for Artificial Intelligence}.
\newblock Harlow, England: Pearson Addison-Wesley, 3 edition.
\newblock ISBN 978-0-201-40375-6.

\bibitem[{Brown et~al.(2020)Brown, Mann, Ryder, Subbiah, Kaplan, Dhariwal,
  Neelakantan, Shyam, Sastry, Askell et~al.}]{brown2020language}
Brown, T.; Mann, B.; Ryder, N.; Subbiah, M.; Kaplan, J.~D.; Dhariwal, P.;
  Neelakantan, A.; Shyam, P.; Sastry, G.; Askell, A.; et~al. 2020.
\newblock Language models are few-shot learners.
\newblock \emph{Advances in neural information processing systems}, 33:
  1877--1901.

\bibitem[{Bubeck et~al.(2023)Bubeck, Chandrasekaran, Eldan, Gehrke, Horvitz,
  Kamar, Lee, Lee, Li, Lundberg et~al.}]{bubeck2023sparks}
Bubeck, S.; Chandrasekaran, V.; Eldan, R.; Gehrke, J.; Horvitz, E.; Kamar, E.;
  Lee, P.; Lee, Y.~T.; Li, Y.; Lundberg, S.; et~al. 2023.
\newblock Sparks of artificial general intelligence: Early experiments with
  gpt-4.
\newblock \emph{arXiv preprint arXiv:2303.12712}.

\bibitem[{Cobbe et~al.(2021)Cobbe, Kosaraju, Bavarian, Chen, Jun, Kaiser,
  Plappert, Tworek, Hilton, Nakano et~al.}]{cobbe2021training}
Cobbe, K.; Kosaraju, V.; Bavarian, M.; Chen, M.; Jun, H.; Kaiser, L.; Plappert,
  M.; Tworek, J.; Hilton, J.; Nakano, R.; et~al. 2021.
\newblock Training verifiers to solve math word problems.
\newblock \emph{arXiv preprint arXiv:2110.14168}.

\bibitem[{De~Raedt, Kimmig, and Toivonen(2007)}]{10.5555/1625275.1625673}
De~Raedt, L.; Kimmig, A.; and Toivonen, H. 2007.
\newblock ProbLog: A Probabilistic Prolog and Its Application in Link
  Discovery.
\newblock In \emph{Proceedings of the 20th International Joint Conference on
  Artifical Intelligence}, IJCAI'07, 2468–2473. San Francisco, CA, USA:
  Morgan Kaufmann Publishers Inc.

\bibitem[{Devlin et~al.(2019)Devlin, Chang, Lee, and
  Toutanova}]{devlin-etal-2019-bert}
Devlin, J.; Chang, M.-W.; Lee, K.; and Toutanova, K. 2019.
\newblock {BERT}: Pre-training of Deep Bidirectional Transformers for Language
  Understanding.
\newblock In Burstein, J.; Doran, C.; and Solorio, T., eds., \emph{Proceedings
  of the 2019 Conference of the North {A}merican Chapter of the Association for
  Computational Linguistics: Human Language Technologies, Volume 1 (Long and
  Short Papers)}, 4171--4186. Minneapolis, Minnesota: Association for
  Computational Linguistics.

\bibitem[{Dries et~al.(2017)Dries, Kimmig, Davis, Belle, and
  de~Raedt}]{ijcai2017-556}
Dries, A.; Kimmig, A.; Davis, J.; Belle, V.; and de~Raedt, L. 2017.
\newblock Solving Probability Problems in Natural Language.
\newblock In \emph{Proceedings of the Twenty-Sixth International Joint
  Conference on Artificial Intelligence, {IJCAI-17}}, 3981--3987.

\bibitem[{Frieder et~al.(2023)Frieder, Pinchetti, Chevalier, Griffiths,
  Salvatori, Lukasiewicz, Petersen, and Berner}]{frieder2023mathematical}
Frieder, S.; Pinchetti, L.; Chevalier, A.; Griffiths, R.-R.; Salvatori, T.;
  Lukasiewicz, T.; Petersen, P.~C.; and Berner, J. 2023.
\newblock Mathematical Capabilities of Chat{GPT}.
\newblock In \emph{Thirty-seventh Conference on Neural Information Processing
  Systems Datasets and Benchmarks Track}.

\bibitem[{Gao et~al.(2023)Gao, Madaan, Zhou, Alon, Liu, Yang, Callan, and
  Neubig}]{gao2023pal}
Gao, L.; Madaan, A.; Zhou, S.; Alon, U.; Liu, P.; Yang, Y.; Callan, J.; and
  Neubig, G. 2023.
\newblock PAL: program-aided language models.
\newblock In \emph{Proceedings of the 40th International Conference on Machine
  Learning}, ICML'23. JMLR.org.

\bibitem[{Google(2023)}]{google2023google}
Google. 2023.
\newblock Google Gemini AI.
\newblock
  \url{https://blog.google/technology/ai/google-gemini-ai/#availability}.

\bibitem[{He-Yueya et~al.(2023)He-Yueya, Poesia, Wang, and
  Goodman}]{heyueya2023solving}
He-Yueya, J.; Poesia, G.; Wang, R.~E.; and Goodman, N.~D. 2023.
\newblock Solving math word problems by combining language models with symbolic
  solvers.
\newblock \emph{arXiv preprint arXiv:2304.09102}.

\bibitem[{Heritage(2013)}]{707988ff-79d1-31c9-afdb-b408b62c1a0f}
Heritage, J. 2013.
\newblock Action formation and its epistemic (and other) backgrounds.
\newblock \emph{Discourse Studies}, 15(5): 551--578.

\bibitem[{Jiang et~al.(2023)Jiang, Sablayrolles, Mensch, Bamford, Chaplot,
  de~las Casas, Bressand, Lengyel, Lample, Saulnier, Lavaud, Lachaux, Stock,
  Scao, Lavril, Wang, Lacroix, and Sayed}]{jiang2023mistral}
Jiang, A.~Q.; Sablayrolles, A.; Mensch, A.; Bamford, C.; Chaplot, D.~S.; de~las
  Casas, D.; Bressand, F.; Lengyel, G.; Lample, G.; Saulnier, L.; Lavaud,
  L.~R.; Lachaux, M.-A.; Stock, P.; Scao, T.~L.; Lavril, T.; Wang, T.; Lacroix,
  T.; and Sayed, W.~E. 2023.
\newblock Mistral 7B.
\newblock arXiv:2310.06825.

\bibitem[{Jin et~al.(2023)Jin, Chen, Leeb, Gresele, Kamal, LYU, Blin, Adauto,
  Kleiman-Weiner, Sachan, and Sch{\"o}lkopf}]{jin2023CLADDER}
Jin, Z.; Chen, Y.; Leeb, F.; Gresele, L.; Kamal, O.; LYU, Z.; Blin, K.; Adauto,
  F.~G.; Kleiman-Weiner, M.; Sachan, M.; and Sch{\"o}lkopf, B. 2023.
\newblock {CL}adder: A Benchmark to Assess Causal Reasoning Capabilities of
  Language Models.
\newblock In \emph{Thirty-seventh Conference on Neural Information Processing
  Systems}.

\bibitem[{Kim et~al.(2023)Kim, Kim, Baek, Bak, and Lee}]{kim-etal-2023-aint}
Kim, J.; Kim, Y.; Baek, I.; Bak, J.; and Lee, J. 2023.
\newblock It Ain{'}t Over: A Multi-aspect Diverse Math Word Problem Dataset.
\newblock In Bouamor, H.; Pino, J.; and Bali, K., eds., \emph{Proceedings of
  the 2023 Conference on Empirical Methods in Natural Language Processing},
  14984--15011. Singapore: Association for Computational Linguistics.

\bibitem[{Koller and Friedman(2009)}]{KollerFriedman2009}
Koller, D.; and Friedman, N. 2009.
\newblock \emph{Probabilistic Graphical Models: Principles and Techniques}.
\newblock The MIT Press, hardcover edition.
\newblock ISBN 9780262013192.

\bibitem[{Landmark, Gulbrandsen, and Svennevig(2015)}]{LANDMARK201554}
Landmark, A. M.~D.; Gulbrandsen, P.; and Svennevig, J. 2015.
\newblock Whose decision? Negotiating epistemic and deontic rights in medical
  treatment decisions.
\newblock \emph{Journal of Pragmatics}, 78: 54--69.
\newblock Epistemics and Deontics in Conversational Directives.

\bibitem[{Mishra et~al.(2022)Mishra, Mitra, Varshney, Sachdeva, Clark, Baral,
  and Kalyan}]{mishra-etal-2022-numglue}
Mishra, S.; Mitra, A.; Varshney, N.; Sachdeva, B.; Clark, P.; Baral, C.; and
  Kalyan, A. 2022.
\newblock {N}um{GLUE}: A Suite of Fundamental yet Challenging Mathematical
  Reasoning Tasks.
\newblock In Muresan, S.; Nakov, P.; and Villavicencio, A., eds.,
  \emph{Proceedings of the 60th Annual Meeting of the Association for
  Computational Linguistics (Volume 1: Long Papers)}, 3505--3523. Dublin,
  Ireland: Association for Computational Linguistics.

\bibitem[{Nafar, Venable, and Kordjamshidi(2024)}]{nafar-etal-2024-teaching}
Nafar, A.; Venable, K.~B.; and Kordjamshidi, P. 2024.
\newblock Teaching Probabilistic Logical Reasoning to Transformers.
\newblock In Graham, Y.; and Purver, M., eds., \emph{Findings of the
  Association for Computational Linguistics: EACL 2024}, 1615--1632. St.
  Julian{'}s, Malta: Association for Computational Linguistics.

\bibitem[{Oaksford and
  Chater(2007)}]{10.1093/acprof:oso/9780198524496.001.0001}
Oaksford, M.; and Chater, N. 2007.
\newblock \emph{{Bayesian Rationality: The probabilistic approach to human
  reasoning}}.
\newblock Oxford University Press.
\newblock ISBN 9780198524496.

\bibitem[{Oaksford and Chater(2009)}]{oaksford_chater_2009}
Oaksford, M.; and Chater, N. 2009.
\newblock Précis of Bayesian Rationality: The Probabilistic Approach to Human
  Reasoning.
\newblock \emph{Behavioral and Brain Sciences}, 32(1): 69–84.

\bibitem[{OpenAI(2023)}]{openai2023GPT4}
OpenAI. 2023.
\newblock GPT-4 Technical Report.
\newblock arXiv:2303.08774.

\bibitem[{Ouyang et~al.(2022)Ouyang, Wu, Jiang, Almeida, Wainwright, Mishkin,
  Zhang, Agarwal, Slama, Ray et~al.}]{ouyang2022training}
Ouyang, L.; Wu, J.; Jiang, X.; Almeida, D.; Wainwright, C.; Mishkin, P.; Zhang,
  C.; Agarwal, S.; Slama, K.; Ray, A.; et~al. 2022.
\newblock Training language models to follow instructions with human feedback.
\newblock \emph{Advances in neural information processing systems}, 35:
  27730--27744.

\bibitem[{Ozturkler et~al.(2023)Ozturkler, Malkin, Wang, and
  Jojic}]{ozturkler-etal-2023-thinksum}
Ozturkler, B.; Malkin, N.; Wang, Z.; and Jojic, N. 2023.
\newblock {T}hink{S}um: Probabilistic reasoning over sets using large language
  models.
\newblock In Rogers, A.; Boyd-Graber, J.; and Okazaki, N., eds.,
  \emph{Proceedings of the 61st Annual Meeting of the Association for
  Computational Linguistics (Volume 1: Long Papers)}, 1216--1239. Toronto,
  Canada: Association for Computational Linguistics.

\bibitem[{Poesia et~al.(2023)Poesia, Gandhi, Zelikman, and
  Goodman}]{poesia2023certified}
Poesia, G.; Gandhi, K.; Zelikman, E.; and Goodman, N.~D. 2023.
\newblock Certified Deductive Reasoning with Language Models.
\newblock arXiv:2306.04031.

\bibitem[{Poggi et~al.(2019)Poggi, D’Errico, Vincze
  et~al.}]{Poggi2019UncertainWU}
Poggi, I.; D’Errico, F.; Vincze, L.; et~al. 2019.
\newblock Uncertain words, uncertain texts. perception and effects of
  uncertainty in biomedical communication.
\newblock \emph{Acta Polytechnica Hungarica}, 16(2): 13--34.

\bibitem[{Rajaby~Faghihi et~al.(2021)Rajaby~Faghihi, Guo, Uszok, Nafar, and
  Kordjamshidi}]{rajaby-faghihi-etal-2021-domiknows}
Rajaby~Faghihi, H.; Guo, Q.; Uszok, A.; Nafar, A.; and Kordjamshidi, P. 2021.
\newblock {D}omi{K}now{S}: A Library for Integration of Symbolic Domain
  Knowledge in Deep Learning.
\newblock In Adel, H.; and Shi, S., eds., \emph{Proceedings of the 2021
  Conference on Empirical Methods in Natural Language Processing: System
  Demonstrations}, 231--241. Online and Punta Cana, Dominican Republic:
  Association for Computational Linguistics.

\bibitem[{Rajaby~Faghihi et~al.(2023)Rajaby~Faghihi, Nafar, Zheng, Mirzaee,
  Zhang, Uszok, Wan, Premsri, Roth, and Kordjamshidi}]{domiknows_glucons}
Rajaby~Faghihi, H.; Nafar, A.; Zheng, C.; Mirzaee, R.; Zhang, Y.; Uszok, A.;
  Wan, A.; Premsri, T.; Roth, D.; and Kordjamshidi, P. 2023.
\newblock GLUECons: A Generic Benchmark for Learning under Constraints.
\newblock \emph{Proceedings of the AAAI Conference on Artificial Intelligence},
  37(8): 9552--9561.

\bibitem[{Rozière et~al.(2024)Rozière, Gehring, Gloeckle, Sootla, Gat, Tan,
  Adi, Liu, Sauvestre, Remez, Rapin, Kozhevnikov, Evtimov, Bitton, Bhatt,
  Ferrer, Grattafiori, Xiong, Défossez, Copet, Azhar, Touvron, Martin,
  Usunier, Scialom, and Synnaeve}]{Llama2024code}
Rozière, B.; Gehring, J.; Gloeckle, F.; Sootla, S.; Gat, I.; Tan, X.~E.; Adi,
  Y.; Liu, J.; Sauvestre, R.; Remez, T.; Rapin, J.; Kozhevnikov, A.; Evtimov,
  I.; Bitton, J.; Bhatt, M.; Ferrer, C.~C.; Grattafiori, A.; Xiong, W.;
  Défossez, A.; Copet, J.; Azhar, F.; Touvron, H.; Martin, L.; Usunier, N.;
  Scialom, T.; and Synnaeve, G. 2024.
\newblock Code Llama: Open Foundation Models for Code.
\newblock arXiv:2308.12950.

\bibitem[{Saeed et~al.(2021)Saeed, Ahmadi, Nakov, and
  Papotti}]{saeed-etal-2021-rulebert}
Saeed, M.; Ahmadi, N.; Nakov, P.; and Papotti, P. 2021.
\newblock {R}ule{BERT}: Teaching Soft Rules to Pre-Trained Language Models.
\newblock In Moens, M.-F.; Huang, X.; Specia, L.; and Yih, S. W.-t., eds.,
  \emph{Proceedings of the 2021 Conference on Empirical Methods in Natural
  Language Processing}, 1460--1476. Online and Punta Cana, Dominican Republic:
  Association for Computational Linguistics.

\bibitem[{Schick et~al.(2024)Schick, Dwivedi-Yu, Dess{\`\i}, Raileanu, Lomeli,
  Hambro, Zettlemoyer, Cancedda, and
  Scialom}]{schick2023toolformerlanguagemodelsteach}
Schick, T.; Dwivedi-Yu, J.; Dess{\`\i}, R.; Raileanu, R.; Lomeli, M.; Hambro,
  E.; Zettlemoyer, L.; Cancedda, N.; and Scialom, T. 2024.
\newblock Toolformer: Language models can teach themselves to use tools.
\newblock \emph{Advances in Neural Information Processing Systems}, 36.

\bibitem[{Shi, Zhang, and Lipani(2022)}]{shi2022stepgame}
Shi, Z.; Zhang, Q.; and Lipani, A. 2022.
\newblock Stepgame: A new benchmark for robust multi-hop spatial reasoning in
  texts.
\newblock In \emph{Proceedings of the AAAI conference on artificial
  intelligence}, volume~36, 11321--11329.

\bibitem[{Suster et~al.(2021)Suster, Fivez, Totis, Kimmig, Davis, de~Raedt, and
  Daelemans}]{suster-etal-2021-mapping}
Suster, S.; Fivez, P.; Totis, P.; Kimmig, A.; Davis, J.; de~Raedt, L.; and
  Daelemans, W. 2021.
\newblock Mapping probability word problems to executable representations.
\newblock In Moens, M.-F.; Huang, X.; Specia, L.; and Yih, S. W.-t., eds.,
  \emph{Proceedings of the 2021 Conference on Empirical Methods in Natural
  Language Processing}, 3627--3640. Online and Punta Cana, Dominican Republic:
  Association for Computational Linguistics.

\bibitem[{Touvron et~al.(2023)Touvron, Martin, Stone, Albert, Almahairi,
  Babaei, Bashlykov, Batra, Bhargava, Bhosale et~al.}]{touvron2023llama}
Touvron, H.; Martin, L.; Stone, K.; Albert, P.; Almahairi, A.; Babaei, Y.;
  Bashlykov, N.; Batra, S.; Bhargava, P.; Bhosale, S.; et~al. 2023.
\newblock Llama 2: Open foundation and fine-tuned chat models.
\newblock \emph{arXiv preprint arXiv:2307.09288}.

\bibitem[{Wei et~al.(2022{\natexlab{a}})Wei, Wang, Schuurmans, Bosma, ichter,
  Xia, Chi, Le, and Zhou}]{NEURIPS2022_9d560961}
Wei, J.; Wang, X.; Schuurmans, D.; Bosma, M.; ichter, b.; Xia, F.; Chi, E.; Le,
  Q.~V.; and Zhou, D. 2022{\natexlab{a}}.
\newblock Chain-of-Thought Prompting Elicits Reasoning in Large Language
  Models.
\newblock In Koyejo, S.; Mohamed, S.; Agarwal, A.; Belgrave, D.; Cho, K.; and
  Oh, A., eds., \emph{Advances in Neural Information Processing Systems},
  volume~35, 24824--24837. Curran Associates, Inc.

\bibitem[{Wei et~al.(2022{\natexlab{b}})Wei, Wang, Schuurmans, Bosma, Xia, Chi,
  Le, Zhou et~al.}]{DBLP:journals/corr/abs-2201-11903}
Wei, J.; Wang, X.; Schuurmans, D.; Bosma, M.; Xia, F.; Chi, E.; Le, Q.~V.;
  Zhou, D.; et~al. 2022{\natexlab{b}}.
\newblock Chain-of-thought prompting elicits reasoning in large language
  models.
\newblock \emph{Advances in neural information processing systems}, 35:
  24824--24837.

\end{thebibliography}

\appendix

\section{Appendix A: Dataset}
\label{appendix:datasetdetails}

\subsection{Generating All Isomorphic DAGs}
\label{appendix:GeneratingAllIsomorphicWeaklyConnectedDirectedGraphs}

Algorithm~\ref{alg:graphgen} is designed to create and collect unique directed acyclic graphs (DAGs) that are also weakly connected. It starts by initializing a set of nodes, labeled from $n_0$ to $n_{n-1}$, where $n$ represents the number of nodes specified by the user. An empty set called $allGraphs$ is also prepared to store the graphs that meet all required conditions.

During each iteration, up to a maximum defined by $maxIter$, the algorithm generates a random permutation of edges between nodes and forms a directed graph from these edges. Each generated graph is then checked for several properties. It must be a DAG and must be weakly connected, ensuring there is a path between any two nodes regardless of edge direction. Additionally, the graph must not be isomorphic to any graph already in $allGraphs$, guaranteeing the uniqueness of each graph in terms of its structure. For a graph to be considered $Arborescence$, it must not have any node with an in-degree greater than one. This enforces a tree-like structure where each node has exactly one parent, except the root which has none.

The variable $maxIter$ can be adjusted in proportion to the number of nodes. For a small number of nodes, specifically when $n < 5$, a few hundred iterations are generally sufficient to discover all possible graphs. However, as the number of nodes increases, even a significantly high $maxIter$ may not uncover all possible graphs. Still, enough distinct types of graphs are typically generated, providing a diverse dataset to sample the final instances.

Finally, after completing the iterations, the algorithm returns the set $allGraphs$, which now contains all the graphs generated that meet the described criteria. This procedure allows for the exploration of various graph configurations within the constraints set by DAG properties, connectivity, and Arborescence (if selected), capturing a diverse set of graph structures. The use of maximum in-degree (MID) in the algorithm's conditions is specifically noted to clarify its role in enforcing the Arborescence condition when required.

\begin{algorithm}
\caption{Generate All Isomorphic Weakly Connected Directed Graphs}
\label{alg:graphgen}
\begin{algorithmic}[1]
\REQUIRE $n$, $maxIter$, $Arborescence$ 
\STATE Initialize $nodes$ with labels $n_0, n_1, \dots, n_{n-1}$
\STATE $allGraphs \gets \emptyset$
\FOR{$iter = 1$ to $maxIter$}
    \STATE $edges \gets$ randomPermutations($nodes, n$)
    \STATE $graph \gets$ directedGraph($edges$)
    \IF{isDAG($graph$) \& isConnected($graph$) \& isNotIsomorphic($graph, allGraphs$) \& Not ($Arborescence$ \& $1 <$ MID($graph$))}
        \STATE add $graph$ to $allGraphs$
    \ENDIF
\ENDFOR
\RETURN $allGraphs$
\STATE \COMMENT{MID is the maximum in-degree of nodes}
\end{algorithmic}
\end{algorithm}

\subsection{Number of Generated Graphs}
\label{appendix:NumberofGeneratedIsomorphicGraphs}
With the $maxIter$ set to $1,000,000$ in Algorithm~\ref{alg:graphgen}, we get the results detailed in Table~\ref{tab:graphsnumbersgenerated}. All possible graphs are generated for the lower number of variables ( like 2, 3, 4, 5, and 6). However, for the larger number of variables, the Algorithm could generate only a subset of the graphs. These numbers are still significant ( more than 100) and would result in a diverse dataset.

\begin{table*}
\begin{center}
\small
\begin{tabular}{|c|c|c|c|c|c|c|c|c|c|}
\hline
Number of nodes & $V_2$ & $V_3$ & $V_4$ & $V_5$ & $V_6$ & $V_7$ & $V_8$ & $V_9$ & $V_{10}$ \\ \hline
Generated Graphs & 1 & 2 & 4 & 9 & 20 & 47 & 101 & 163 & 116 \\
Number of Isomorphic Graphs & 1 & 2 & 4 & 9 & 20 & 48 & 115 & 286 & 719 \\

\hline
\end{tabular}
\caption{Number of Isomorphic Graphs generated by our Algorithm and all the possible Isomorphic Graphs.}
\label{tab:graphsnumbersgenerated}
\end{center}
\end{table*}

\begin{comment}

\subsection{CPT Entries in Arborescence BNs}
\label{appendix:CPTEntriesinArborescenceBNs}

\begin{lemma}
\label{lemma:CPTinArborescence}
The total number of entries in all CPTs of an Arborescence BN is calculated by the formula: $2 + (V - 1) \times 4$
\end{lemma}

\begin{proof}
First, consider the root node of the arborescence. Since it has no parents, its CPT needs only to define the probability distribution of the root node itself. Assuming the root node and all other nodes are binary, the CPT for the root will contain exactly two entries: probabilities of the root node being true or false.

For each of the other $V-1$ nodes, which are not roots and have exactly one parent, the CPT must account for the conditional probabilities based on the states of their binary parent. For each state of the parent, the child can either be true or false, thus requiring four entries per non-root node. These entries account for each combination of parent and child states.

In $2 + (V - 1) \times 4$, $2$ accounts for the two entries of the root node's CPT, and the $ (V - 1) \times 4$ accounts for the four entries for each of the $V-1$ non-root nodes. This confirms that for an arborescence graph with $V$ variables, there are exactly $2 + (V-1) \times 4$ probability entries in the BN's CPTs.
\end{proof}
\end{comment}

\subsection{Number of Queries in a BN}
\label{appendix:NumberofQueriesinaBN}
We generate all the possible queries and then select the complex ones.

\begin{lemma}
\label{lemma:NumberofQueries}
For a Bayesian network with $V$ nodes, there are $3^V - 2^V$ ways to define a query.
\end{lemma}

\begin{proof}
Consider a Bayesian network with $V$ nodes. Each node can be in one of three states: dependent variables, conditioning variables, or not included in the query. This setup leads to $3^V$ total configurations. However queries that do not include at least one dependent variable are not valid. In this queries every variable is either a conditioning variables, or not included in the query leading to $2^V$ configurations. As a result, when we exclude these invalid queries we end up with $3^V - 2^V$ ways to define a query.
\end{proof}

\subsection{Pipeline}
\label{appendix:datasetpipeline}

The complete pipeline of our Dataset Creation is depicted in Figure~\ref{fig:datasetpipeline}. It begins with the generation of Bayesian Networks by creating all possible isomorphic weakly connected directed graphs with node counts ranging from 2 to 10. CPTs for these BNs are then filled with uniformly distributed random probabilities. Once the BNs and CPTs are established, probabilistic queries are generated for each BN. Following the generation of the BNs and queries, textual contexts and queries are constructed. These contexts translate the CPTs into natural language. Similarly, the queries are also translated into simple, structured textual forms.

The final stage of the dataset creation involves validating and using the BNs, probabilities, and queries through the Python library pgmpy. This library performs exact inference to determine the correct answers to the probabilistic queries, ensuring the dataset instance is ready for use. At the end, the dataset instance would comprise the textual context, question and the probabilistic answer.

\begin{figure*}
    \centering
    \includegraphics[width=0.80\linewidth]{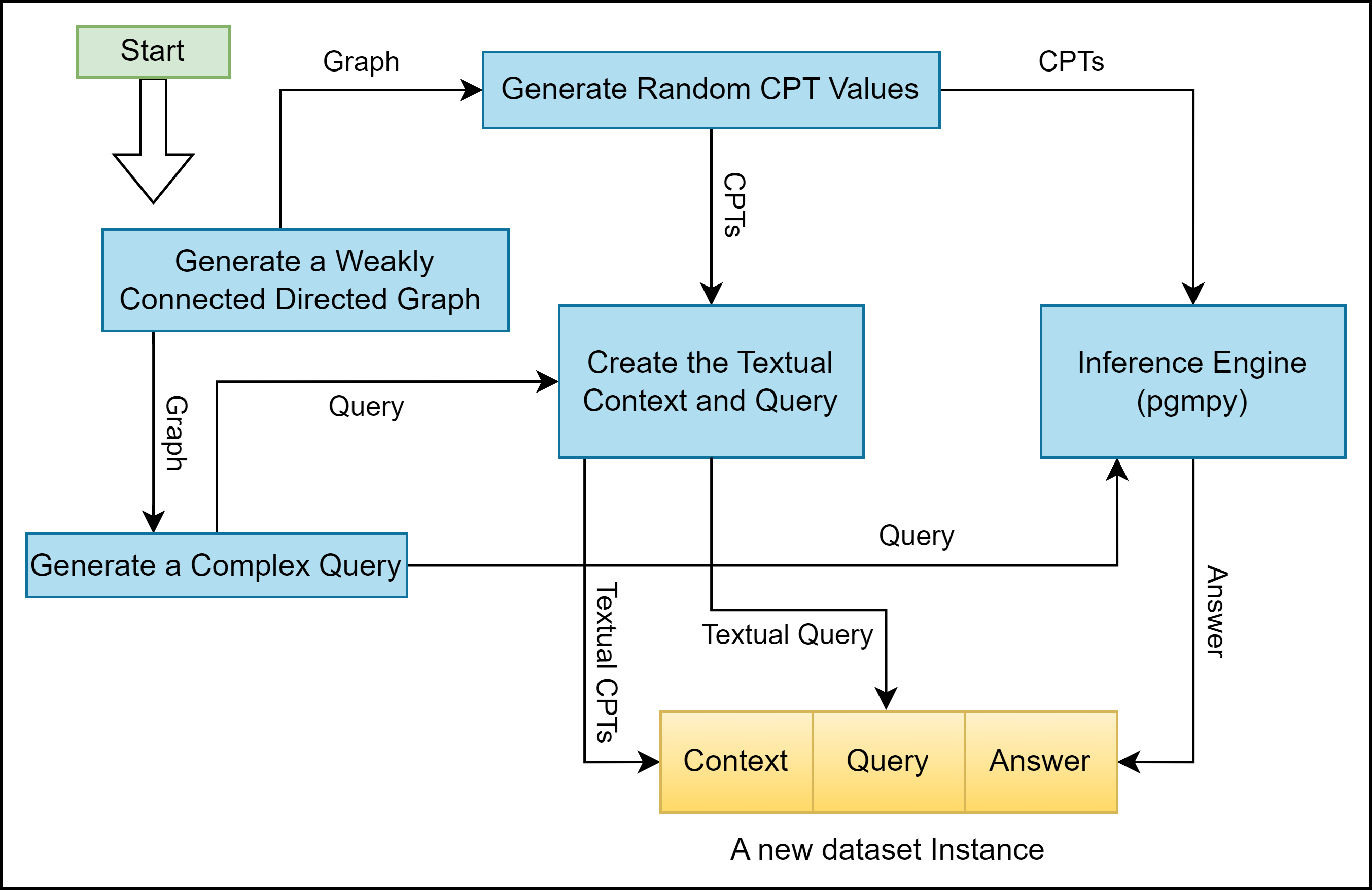}
\caption{Pipeline of the dataset creation. This includes the creation of the BN, CPTs, and query. CPT and query are used to create the textual context and question and to infer the final answer to create a dataset instance.}

\label{fig:datasetpipeline}
\end{figure*}

\section{Appendix B: Hyper-parameters} 

In testing our GPT models, we use the OpenAI API and set ``max\_tokens'' and ``temperature'' parameters to $1500$ and $0.2$ for both GPT3.5 and GPT4. The specific version for these models are ``gpt-3.5-turbo-0613'' and ``gpt-4-0613''. For Llama3, we use ``meta-llama-3-70b-instruct'' with the ``temperature'' set to $0.00$, ``max\_new\_tokens'' set to $2000$ and ``top\_p'' set to $0.95$. For the Other Open-source models, we use ``Llama2-70B-chat'' and ``mistral-7b-instruct-v0.2''. In these models we set ``temperature'' to $0.01$, ``max\_new\_tokens'' to $2000$ and ``top\_p'' to $0.999$.

The in-context examples that we used for all our methods and whenever we included few-shot testing are shown in Figure~\ref{fig:fsexamples}. Their difficulty increases as we move from examples one to three with variable sizes of 2, 3, and 4. This example selection maintains a balance between the difficulty of the in-context examples and the length of the solutions while showing different queries and methods to solve these examples. It is essential to select examples that maintain this diversity, as we have observed that the accuracy of GPT-3.5 heavily depends on the examples chosen. For example, if an event name was repeated too many times (like in all our selected examples), that would lead GPT3.5 to believe that the repeated event should be part of every solution regardless of the context! GPT4 on the other hand, was very robust, even with bad in-context examples. Regardless, in the end, they were tested with the same examples.

\section{Appendix C: Additional Experimentation}
\label{appendix:additionalresults}

\subsection{Other LLMs}
\label{appendix:opensourceresults}
The comprehensive results of Mistral and Llama2 tested on the BLInD dataset are shown in Table~\ref{tab:BQACOTLlama} next to Llama3 results. Our results show that these models are not strong enough to conduct probabilistic reasoning on our dataset. This weakness stems from both the reasoning capability and the ability to follow a long structured solution. For example, in the Monte Carlo method, these methods could not imitate the format and would produce answers that don't follow the instructed format of the solution. On the other hand, in the PAL method, these methods could follow the structure provided in the instruction and in-context examples. However, the final answer was still incorrect. 

In addition to the LLM version mentioned in Appendix B, namely ''Mistral -7b-instruct-v0.2`` and ''llama-2-70b-chat``, we also tried ''codellama-70b``~\cite{Llama2024code} and its variants ''codellama-70b-python`` and ''codellama-70b-instruct`` in attempt to improve the results of PAL, Monte Carlo and ProbLog. Ultimately, we decided to test these methods with ''llama-2-70b-chat`` since the results did not improve with other variants.

\begin{table*}[!t]
\begin{center}
\small
\begin{tabular}{|c|c|P{0.4cm}|P{0.4cm}|P{0.4cm}|P{0.4cm}|P{0.4cm}|P{0.4cm}|P{0.4cm}|P{0.4cm}|P{0.4cm}||P{0.95cm}|P{0.95cm}|P{0.95cm}|}
\hline
Model &Method & $V_2$ & $V_3$ & $V_4$ & $V_5$ & $V_6$ & $V_7$ & $V_8$ & $V_9$ & $V_{10}$ & $V_{2-5}$ & $V_{6-10}$ & $V_{2-10}$ \\
\hline
\multirow{ 7}{*}{Mistral} &Basic QA ZS  & \HeatCell{26} & \HeatCell{6} & \HeatCell{0} & \HeatCell{2} & \HeatCell{2} & \HeatCell{3} & \HeatCell{1} & \HeatCell{1} & \HeatCell{0} & \HeatCell{8} & \HeatCell{1} & \HeatCell{4}  \\
&Basic QA FS  & \HeatCell{14} & \HeatCell{7} & \HeatCell{4} & \HeatCell{2} & \HeatCell{4} & \HeatCell{2} & \HeatCell{1} & \HeatCell{1} & \HeatCell{0} & \HeatCell{6} & \HeatCell{1} & \HeatCell{3}  \\
&COT ZS  & \HeatCell{21} & \HeatCell{8} & \HeatCell{1} & \HeatCell{5} & \HeatCell{2} & \HeatCell{2} & \HeatCell{3} & \HeatCell{2} & \HeatCell{2} & \HeatCell{8} & \HeatCell{2} & \HeatCell{5}  \\
&COT FS  & \HeatCell{28} & \HeatCell{11} & \HeatCell{4} & \HeatCell{1} & \HeatCell{3} & \HeatCell{1} & \HeatCell{2} & \HeatCell{3} & \HeatCell{2} & \HeatCell{11} & \HeatCell{2} & \HeatCell{6}  \\
&PAL  & \HeatCell{4} & \HeatCell{0} & \HeatCell{1} & \HeatCell{1} & \HeatCell{2} & \HeatCell{2} & \HeatCell{1} & \HeatCell{1} & \HeatCell{0} & \HeatCell{1} & \HeatCell{1} & \HeatCell{1}  \\
&Monte Carlo  & \HeatCell{4} & \HeatCell{0} & \HeatCell{1} & \HeatCell{1} & \HeatCell{2} & \HeatCell{2} & \HeatCell{1} & \HeatCell{1} & \HeatCell{0} & \HeatCell{1} & \HeatCell{1} & \HeatCell{1}  \\
&ProbLog  & \HeatCell{3} & \HeatCell{0} & \HeatCell{1} & \HeatCell{1} & \HeatCell{2} & \HeatCell{2} & \HeatCell{1} & \HeatCell{1} & \HeatCell{0} & \HeatCell{1} & \HeatCell{1} & \HeatCell{1}  \\

%NE  & \HeatCell{0} & \HeatCell{0} & \HeatCell{0} & \HeatCell{0} & \HeatCell{0} & \HeatCell{0} & \HeatCell{0} & \HeatCell{0} & \HeatCell{0} & \HeatCell{0} & \HeatCell{0} & \HeatCell{0}  \\
%GG  & \HeatCell{0} & \HeatCell{0} & \HeatCell{0} & \HeatCell{0} & \HeatCell{0} & \HeatCell{0} & \HeatCell{0} & \HeatCell{0} & \HeatCell{0} & \HeatCell{0} & \HeatCell{0} & \HeatCell{0}  \\

\hline
\multirow{ 7}{*}{Llama2} & Basic QA ZS & \HeatCell{25} & \HeatCell{11} & \HeatCell{2} & \HeatCell{3} & \HeatCell{3} & \HeatCell{2} & \HeatCell{4} & \HeatCell{4} & \HeatCell{0} & \HeatCell{10} & \HeatCell{2} & \HeatCell{6}  \\
&Basic QA FS & \HeatCell{3} & \HeatCell{0} & \HeatCell{1} & \HeatCell{1} & \HeatCell{2} & \HeatCell{2} & \HeatCell{1} & \HeatCell{1} & \HeatCell{0} & \HeatCell{1} & \HeatCell{1} & \HeatCell{1}  \\
&COT ZS & \HeatCell{8} & \HeatCell{2} & \HeatCell{1} & \HeatCell{6} & \HeatCell{1} & \HeatCell{1} & \HeatCell{3} & \HeatCell{1} & \HeatCell{2} & \HeatCell{4} & \HeatCell{1} & \HeatCell{2}  \\
&COT FS & \HeatCell{2} & \HeatCell{1} & \HeatCell{3} & \HeatCell{2} & \HeatCell{1} & \HeatCell{2} & \HeatCell{2} & \HeatCell{4} & \HeatCell{1} & \HeatCell{2} & \HeatCell{2} & \HeatCell{2}  \\
&PAL  & \HeatCell{16} & \HeatCell{1} & \HeatCell{3} & \HeatCell{2} & \HeatCell{2} & \HeatCell{2} & \HeatCell{1} & \HeatCell{1} & \HeatCell{0} & \HeatCell{5} & \HeatCell{1} & \HeatCell{3}  \\
&Monte Carlo  & \HeatCell{10} & \HeatCell{11} & \HeatCell{4} & \HeatCell{4} & \HeatCell{3} & \HeatCell{2} & \HeatCell{1} & \HeatCell{1} & \HeatCell{0} & \HeatCell{7} & \HeatCell{1} & \HeatCell{4}  \\
&ProbLog  & \HeatCell{3} & \HeatCell{7} & \HeatCell{3} & \HeatCell{1} & \HeatCell{2} & \HeatCell{1} & \HeatCell{1} & \HeatCell{1} & \HeatCell{0} & \HeatCell{3} & \HeatCell{1} & \HeatCell{2}  \\

\hline
\multirow{ 7}{*}{Llama3} & Basic QA ZS & \HeatCell{31} & \HeatCell{21} & \HeatCell{5} & \HeatCell{6} & \HeatCell{6} & \HeatCell{5} & \HeatCell{1} & \HeatCell{1} & \HeatCell{0} & \HeatCell{16} & \HeatCell{3} & \HeatCell{8}  \\
&Basic QA FS & \HeatCell{3} & \HeatCell{0} & \HeatCell{1} & \HeatCell{1} & \HeatCell{2} & \HeatCell{2} & \HeatCell{1} & \HeatCell{1} & \HeatCell{0} & \HeatCell{1} & \HeatCell{1} & \HeatCell{1}  \\
&COT ZS & \HeatCell{63} & \HeatCell{45} & \HeatCell{21} & \HeatCell{17} & \HeatCell{18} & \HeatCell{11} & \HeatCell{9} & \HeatCell{4} & \HeatCell{2} & \HeatCell{37} & \HeatCell{9} & \HeatCell{21}  \\
&COT FS & \HeatCell{63} & \HeatCell{46} & \HeatCell{21} & \HeatCell{12} & \HeatCell{20} & \HeatCell{15} & \HeatCell{7} & \HeatCell{8} & \HeatCell{5} & \HeatCell{36} & \HeatCell{11} & \HeatCell{22}  \\
&PAL  & \HeatCell{100} & \HeatCell{84} & \HeatCell{57} & \HeatCell{36} & \HeatCell{31} & \HeatCell{20} & \HeatCell{10} & \HeatCell{14} & \HeatCell{8} & \HeatCell{69} & \HeatCell{17} & \HeatCell{40}  \\
&PAL w/NE  & \HeatCell{100} & \HeatCell{95} & \HeatCell{71} & \HeatCell{52} & \HeatCell{46} & \HeatCell{28} & \HeatCell{16} & \HeatCell{16} & \HeatCell{9} & \HeatCell{79} & \HeatCell{23} & \HeatCell{48}  \\
&Monte Carlo  & \HeatCell{100} & \HeatCell{100} & \HeatCell{96} & \HeatCell{96} & \HeatCell{92} & \HeatCell{85} & \HeatCell{77} & \HeatCell{72} & \HeatCell{64} & \HeatCell{98} & \HeatCell{78} & \HeatCell{87}  \\
&Monte Carlo w/GG & \HeatCell{100} & \HeatCell{100} & \HeatCell{99} & \HeatCell{93} & \HeatCell{91} & \HeatCell{84} & \HeatCell{70} & \HeatCell{52} & \HeatCell{51} & \HeatCell{98} & \HeatCell{70} & \HeatCell{82}  \\
&ProbLog  & \HeatCell{90} & \HeatCell{95} & \HeatCell{92} & \HeatCell{87} & \HeatCell{95} & \HeatCell{94} & \HeatCell{87} & \HeatCell{82} & \HeatCell{78} & \HeatCell{91} & \HeatCell{87} & \HeatCell{89}  \\

\hline
\end{tabular}
\end{center}
\caption{Mistral, Llama2 and Llama3 performance based on different methods on BLInD. w/NE, w/GG, ZS and FS denote utilization of \textit{Number Extraction}, \textit{Graph Generation}, zero-shot and few-shot, respectively.}
\label{tab:BQACOTLlama}
\end{table*}

\begin{table}
\begin{center}
    
\small
\begin{tabular}{|c|c|c|c|}
\hline
Method & GPT3.5 & Llama3 & GPT4 \\
\hline
\hline
Basic QA ZS & \HeatCell{0} & \HeatCell{0} & \HeatCell{20} \\
Basic QA FS & \HeatCell{0} & \HeatCell{0} & \HeatCell{0} \\
COT ZS & \HeatCell{9} & \HeatCell{47} & \HeatCell{65} \\
COT FS & \HeatCell{3} & \HeatCell{38} & \HeatCell{64} \\
\hline
PAL & \HeatCell{26} & \HeatCell{91} & \HeatCell{97} \\
PAL w/GG & \HeatCell{28} & \HeatCell{75} & \HeatCell{97} \\
PAL w/NE & \HeatCell{39} & \HeatCell{96} & \HeatCell{96} \\
PAL w/NE w/GG & \HeatCell{41} & \HeatCell{98} & \HeatCell{97} \\
\hline
Monte Carlo & \HeatCell{75} & \HeatCell{96} & \HeatCell{98} \\
Monte Carlo w/GG & \HeatCell{75} & \HeatCell{95} & \HeatCell{97} \\
Monte Carlo w/NE & \HeatCell{63} & \HeatCell{90} & \HeatCell{98} \\
Monte Carlo w/NE W/G & \HeatCell{70} & \HeatCell{91} & \HeatCell{97} \\
\hline
ProbLog & \HeatCell{71} & \HeatCell{84} & \HeatCell{97} \\
%ProbLog w/NE & \HeatCell{53} & \HeatCell{0} & \HeatCell{94} \\
%ProbLog w/NE w/GG & \HeatCell{56} & \HeatCell{0} & \HeatCell{95} \\
\hline
\end{tabular}
\caption{Complete results on the adapted version of CLADDER dataset. w/NE and w/GG stand for with \textit{Number Extraction} and with \textit{Graph Generation}.}
\label{tab:compeleteCLADDER}
\end{center}
\end{table}

\begin{table*}
\begin{center}
\small
\begin{tabular}{|c|c|c|c|c|c|c|c|c|c|c|c|c|}
\hline
Method & $V_2$ & $V_3$ & $V_4$ & $V_5$ & $V_6$ & $V_7$ & $V_8$ & $V_9$ & $V_{10}$ & $V_{2-5}$ & $V_{6-10}$ & $V_{2-10}$ \\

\hline
\hline

Basic QA ZS  & \HeatCell{33} & \HeatCell{13} & \HeatCell{5} & \HeatCell{4} & \HeatCell{6} & \HeatCell{2} & \HeatCell{3} & \HeatCell{1} & \HeatCell{2} & \HeatCell{13} & \HeatCell{2} & \HeatCell{7}  \\

COT ZS  & \HeatCell{53} & \HeatCell{8} & \HeatCell{4} & \HeatCell{5} & \HeatCell{10} & \HeatCell{5} & \HeatCell{2} & \HeatCell{2} & \HeatCell{0} & \HeatCell{17} & \HeatCell{3} & \HeatCell{9}  \\
COT FS  & \HeatCell{52} & \HeatCell{23} & \HeatCell{12} & \HeatCell{5} & \HeatCell{8} & \HeatCell{4} & \HeatCell{1} & \HeatCell{4} & \HeatCell{2} & \HeatCell{23} & \HeatCell{3} & \HeatCell{12}  \\
PAL & \HeatCell{66} & \HeatCell{34} & \HeatCell{25} & \HeatCell{17} & \HeatCell{14} & \HeatCell{9} & \HeatCell{6} & \HeatCell{5} & \HeatCell{2} & \HeatCell{35} & \HeatCell{7} & \HeatCell{19}  \\
PAL w/NE  & \HeatCell{85} & \HeatCell{66} & \HeatCell{41} & \HeatCell{27} & \HeatCell{19} & \HeatCell{12} & \HeatCell{5} & \HeatCell{3} & \HeatCell{6} & \HeatCell{54} & \HeatCell{9} & \HeatCell{29}  \\
Monte Carlo  & \HeatCell{79} & \HeatCell{63} & \HeatCell{71} & \HeatCell{65} & \HeatCell{41} & \HeatCell{32} & \HeatCell{33} & \HeatCell{18} & \HeatCell{14} & \HeatCell{69} & \HeatCell{27} & \HeatCell{46}  \\
Monte Carlo w/GG  & \HeatCell{85} & \HeatCell{82} & \HeatCell{83} & \HeatCell{68} & \HeatCell{42} & \HeatCell{31} & \HeatCell{28} & \HeatCell{18} & \HeatCell{8} & \HeatCell{79} & \HeatCell{25} & \HeatCell{49}  \\
ProbLog  & \HeatCell{87} & \HeatCell{82} & \HeatCell{88} & \HeatCell{75} & \HeatCell{59} & \HeatCell{52} & \HeatCell{46} & \HeatCell{38} & \HeatCell{35} & \HeatCell{83} & \HeatCell{46} & \HeatCell{62}  \\

\hline
\hline

Basic QA FS  & \HeatCell{3} & \HeatCell{0} & \HeatCell{1} & \HeatCell{1} & \HeatCell{2} & \HeatCell{2} & \HeatCell{1} & \HeatCell{1} & \HeatCell{0} & \HeatCell{1} & \HeatCell{1} & \HeatCell{1}  \\
PAL w/NE w/GG  & \HeatCell{82} & \HeatCell{56} & \HeatCell{39} & \HeatCell{30} & \HeatCell{20} & \HeatCell{11} & \HeatCell{3} & \HeatCell{2} & \HeatCell{4} & \HeatCell{51} & \HeatCell{7} & \HeatCell{27}  \\
Monte Carlo w/NE  & \HeatCell{96} & \HeatCell{83} & \HeatCell{66} & \HeatCell{48} & \HeatCell{30} & \HeatCell{21} & \HeatCell{11} & \HeatCell{10} & \HeatCell{3} & \HeatCell{73} & \HeatCell{15} & \HeatCell{40}  \\
Monte Carlo w/NE w/GG  & \HeatCell{96} & \HeatCell{88} & \HeatCell{67} & \HeatCell{55} & \HeatCell{37} & \HeatCell{31} & \HeatCell{15} & \HeatCell{9} & \HeatCell{10} & \HeatCell{76} & \HeatCell{20} & \HeatCell{45}  \\
ProbLog w/NE & \HeatCell{85} & \HeatCell{87} & \HeatCell{40} & \HeatCell{32} & \HeatCell{27} & \HeatCell{35} & \HeatCell{25} & \HeatCell{23} & \HeatCell{12} & \HeatCell{61} & \HeatCell{24} & \HeatCell{40}  \\
ProbLog w/NE w/GG & \HeatCell{89} & \HeatCell{87} & \HeatCell{49} & \HeatCell{40} & \HeatCell{28} & \HeatCell{25} & \HeatCell{15} & \HeatCell{20} & \HeatCell{16} & \HeatCell{66} & \HeatCell{20} & \HeatCell{41}  \\

\hline
\end{tabular}
\end{center}
\caption{All the methods and configurations tested on GPT3.5 in the order that they improved the model, with the unhelpful configurations separated at the bottom. w/NE and w/GG stand for with \textit{Number Extraction} and with \textit{Graph Generation}.}
\label{tab:ALLGPT3Results}
\end{table*}

\begin{table*}
\begin{center}
\small
\begin{tabular}{|c|c|c|c|c|c|c|c|c|c|c|c|}
\hline
Model & Subtask & Method  & $V_2$ & $V_3$ & $V_4$ & $V_5$ & $V_6$ & $V_7$ & $V_8$ & $V_9$ & $V_{10}$ \\
\hline
\multirow{ 16}{*}{GPT3.5}& NE & NE & \HeatCell{100} & \HeatCell{100} & \HeatCell{100} & \HeatCell{100} & \HeatCell{96} & \HeatCell{95} & \HeatCell{98} & \HeatCell{94} & \HeatCell{94} \\
\cline{2-12}
&GG & GG & \HeatCell{100} & \HeatCell{95} & \HeatCell{92} & \HeatCell{93} & \HeatCell{84} & \HeatCell{75} & \HeatCell{79} & \HeatCell{73} & \HeatCell{78} \\
\cline{2-12}
&NE & NE then GG & \HeatCell{100} & \HeatCell{100} & \HeatCell{100} & \HeatCell{100} & \HeatCell{97} & \HeatCell{95} & \HeatCell{95} & \HeatCell{94} & \HeatCell{95} \\
&GG & NE then GG & \HeatCell{100} & \HeatCell{100} & \HeatCell{99} & \HeatCell{96} & \HeatCell{92} & \HeatCell{87} & \HeatCell{78} & \HeatCell{70} & \HeatCell{76} \\
\cline{2-12}
&NE & GG then NE & \HeatCell{100} & \HeatCell{98} & \HeatCell{99} & \HeatCell{98} & \HeatCell{98} & \HeatCell{95} & \HeatCell{95} & \HeatCell{89} & \HeatCell{84} \\
&GG & GG and NE & \HeatCell{100} & \HeatCell{97} & \HeatCell{100} & \HeatCell{96} & \HeatCell{92} & \HeatCell{87} & \HeatCell{87} & \HeatCell{77} & \HeatCell{82} \\
\cline{2-12}
\cline{2-12}
&NE & PAL w/NE & \HeatCell{98} & \HeatCell{93} & \HeatCell{66} & \HeatCell{58} & \HeatCell{36} & \HeatCell{36} & \HeatCell{32} & \HeatCell{17} & \HeatCell{17} \\
&NE & PAL w/NE w/GG & \HeatCell{88} & \HeatCell{90} & \HeatCell{54} & \HeatCell{52} & \HeatCell{31} & \HeatCell{31} & \HeatCell{19} & \HeatCell{12} & \HeatCell{8} \\
&GG & PAL w/NE w/GG & \HeatCell{98} & \HeatCell{100} & \HeatCell{97} & \HeatCell{83} & \HeatCell{69} & \HeatCell{64} & \HeatCell{44} & \HeatCell{26} & \HeatCell{18} \\
\cline{2-12}
&NE & Monte Carlo w/NE & \HeatCell{97} & \HeatCell{94} & \HeatCell{79} & \HeatCell{61} & \HeatCell{46} & \HeatCell{43} & \HeatCell{35} & \HeatCell{29} & \HeatCell{20} \\
&GG & Monte Carlo w/GG & \HeatCell{100} & \HeatCell{98} & \HeatCell{97} & \HeatCell{79} & \HeatCell{61} & \HeatCell{64} & \HeatCell{64} & \HeatCell{58} & \HeatCell{56} \\
&NE & Monte Carlo w/NE w/GG & \HeatCell{95} & \HeatCell{96} & \HeatCell{79} & \HeatCell{61} & \HeatCell{53} & \HeatCell{49} & \HeatCell{38} & \HeatCell{32} & \HeatCell{25} \\
&GG & Monte Carlo w/NE w/GG & \HeatCell{100} & \HeatCell{99} & \HeatCell{96} & \HeatCell{79} & \HeatCell{73} & \HeatCell{63} & \HeatCell{47} & \HeatCell{30} & \HeatCell{21} \\
\cline{2-12}
&NE & ProbLog w/NE & \HeatCell{93} & \HeatCell{96} & \HeatCell{65} & \HeatCell{50} & \HeatCell{38} & \HeatCell{42} & \HeatCell{29} & \HeatCell{24} & \HeatCell{24} \\
&NE & ProbLog w/NE w/GG & \HeatCell{86} & \HeatCell{94} & \HeatCell{58} & \HeatCell{46} & \HeatCell{31} & \HeatCell{32} & \HeatCell{17} & \HeatCell{18} & \HeatCell{13} \\
&GG & ProbLog w/NE w/GG & \HeatCell{100} & \HeatCell{100} & \HeatCell{100} & \HeatCell{85} & \HeatCell{80} & \HeatCell{70} & \HeatCell{60} & \HeatCell{33} & \HeatCell{31} \\

\hline
\hline

\multirow{ 2}{*}{Llama3}& NE & PAL w/NE & \HeatCell{100} & \HeatCell{100} & \HeatCell{99} & \HeatCell{100} & \HeatCell{100} & \HeatCell{100} & \HeatCell{95} & \HeatCell{93} & \HeatCell{97} \\

&GG & Monte Carlo w/GG & \HeatCell{100} & \HeatCell{98} & \HeatCell{97} & \HeatCell{79} & \HeatCell{61} & \HeatCell{64} & \HeatCell{64} & \HeatCell{58} & \HeatCell{56} \\
\hline
\end{tabular}
\caption{Graph Generation (GG) and Number Extraction (NE) subtask accuraices when used in the methods. At the top of the table, NE and GG are evaluated in isolation and then combined together. At the bottom of the table, they are evaluated when used within our methods.}
\label{tab:compeleteNEGG}
\end{center}
\end{table*}

\begin{table*}
\begin{center}
\small
\begin{tabular}{|c|c|c|c|c|c|c|c|c|c|}
\hline
Method & $V_2$ & $V_3$ & $V_4$ & $V_5$ & $V_6$ & $V_7$ & $V_8$ & $V_9$ & $V_{10}$ \\
\hline
PAL & \HeatCell{100} & \HeatCell{97} & \HeatCell{97} & \HeatCell{97} & \HeatCell{95} & \HeatCell{94} & \HeatCell{94} & \HeatCell{93} & \HeatCell{90} \\
PAL w/NE & \HeatCell{100} & \HeatCell{100} & \HeatCell{100} & \HeatCell{98} & \HeatCell{96} & \HeatCell{90} & \HeatCell{91} & \HeatCell{89} & \HeatCell{88} \\
PAL w/NE w/GG & \HeatCell{99} & \HeatCell{100} & \HeatCell{99} & \HeatCell{93} & \HeatCell{93} & \HeatCell{90} & \HeatCell{81} & \HeatCell{87} & \HeatCell{78} \\
\hline
Monte Carlo & \HeatCell{100} & \HeatCell{100} & \HeatCell{98} & \HeatCell{100} & \HeatCell{96} & \HeatCell{96} & \HeatCell{94} & \HeatCell{89} & \HeatCell{79} \\
Monte Carlo w/NE & \HeatCell{97} & \HeatCell{100} & \HeatCell{99} & \HeatCell{95} & \HeatCell{92} & \HeatCell{85} & \HeatCell{77} & \HeatCell{66} & \HeatCell{59} \\
Monte Carlo w/GG & \HeatCell{100} & \HeatCell{100} & \HeatCell{100} & \HeatCell{100} & \HeatCell{93} & \HeatCell{83} & \HeatCell{74} & \HeatCell{69} & \HeatCell{67} \\
Monte Carlo w/NE W/G & \HeatCell{100} & \HeatCell{100} & \HeatCell{98} & \HeatCell{97} & \HeatCell{92} & \HeatCell{90} & \HeatCell{86} & \HeatCell{66} & \HeatCell{69} \\
\hline
ProbLog & \HeatCell{99} & \HeatCell{98} & \HeatCell{94} & \HeatCell{90} & \HeatCell{87} & \HeatCell{79} & \HeatCell{80} & \HeatCell{72} & \HeatCell{62} \\
ProbLog w/NE & \HeatCell{97} & \HeatCell{97} & \HeatCell{50} & \HeatCell{35} & \HeatCell{29} & \HeatCell{37} & \HeatCell{30} & \HeatCell{28} & \HeatCell{15} \\
ProbLog w/NE w/GG & \HeatCell{94} & \HeatCell{90} & \HeatCell{54} & \HeatCell{48} & \HeatCell{32} & \HeatCell{29} & \HeatCell{22} & \HeatCell{23} & \HeatCell{18} \\

\hline
\end{tabular}
\caption{GPT3.5 success rate of code execution with different methods and number of variables. These results only consider the syntax accuracy of a code not its solution accuracy. w/NE and w/GG stand for with \textit{Number Extraction} and with \textit{Graph Generation}.}
\label{tab:ALLerros}
\end{center}
\end{table*}

\subsection{All Configuration Results}
\label{appendix:gpt3NEGGALL}
We tested many configurations of \textit{Number Extraction} and \textit{Graph Generation} in combination with our three main PAL, Monte Carlo, and ProbLog methods. Table~\ref{tab:ALLGPT3Results} shows all the methods and configurations tested on GPT3.5 in the order that they improved the model, with the unhelpful configurations separated at the bottom. The methods that were not helpful include 1) Basic QA FS, which, as we discussed before, seems only to confuse the model. 2) Monte Carlo and ProbLog with w/NE, which is not helpful since they do not add any meaningful contribution to the particular problem while they lengthen the solution unnecessarily. 3) Any method with both subtasks included in the prompt, results in a very long solution that confuses the LLM. For Llama3, we did not try every possible configuration based on our findings from GPT3.5. However, Monte Carlo w/GG, which underperformed in Llama3, is shown in Table~\ref{tab:BQACOTLlama}.

\subsection{Accuracy Analysis of subtasks}
\label{appendix:NEGG}
Another avenue that we extensively investigated was the accuracy of subtasks when utilized in our methods. As we mentioned before, the accuracy of subtasks varies compared to when they are generated before the main solution within the same prompt. Their accuracy with the combination of our methods is shown in Table~\ref{tab:compeleteNEGG}. The first column determines the model used, which is either GPT3.5 (for most of the tests) or Llama3. The second column, ``Subtask'', determines the task on which we report our accuracy: NE (Number Extraction) or GG (Graph Generation). The ``Method'' column shows the configuration in which the task is evaluated, and the rest of the columns show the number of variables in our splits $V_i$s. As shown in this table, the accuracy of both subtasks drops dramatically as the number of variables increases. It is worth noting that in most of the wrong examples, only one edge or number is incorrect. However, it is still important for the LLM to extract everything correctly since our datasets require all the numbers for inference. Even one incorrect number will lead to an incorrect result. When we compare the \textit{Number Extraction} accuracy with the main task accuracy, we see that \textit{Number Extraction} accuracy imposes an upper limit on the accuracy of the main task. Another thing we observed was that when the number of variables becomes too large, the LLM will not follow the format correctly, leading to many problems with larger variables.

\subsection{Adapted CLADDER Dataset}
\label{appendix:AdaptedCLADDERDataset}
We tested all our methods and configurations on the adapted version of the CLADDER dataset. These results are shown in Table~\ref{tab:compeleteCLADDER}. GPT3.5 quickly reaches 97\% accuracy, which shows the need to create more challenging datasets in this field, such as BLInD. Surprisingly, PAL w/NE w/GG results in a slight improvement in GPT3.5, which shows that if the solution length becomes smaller, additional information becomes more straightforward for LLM to utilize. Still, the improvements are minor as these subtasks become more useful at with larger variables which CLADDER does not have. Llama3 ProbLog still outperforms for syntax issues.

\section{Appendix D: Code Generation errors}
\label{appendix:codeerror}
We discussed earlier GPT3.5's struggle with coding and the coding-specific challenges that are introduced when we utilize our main methods. Table~\ref{tab:ALLerros} shows the success rate of code execution of GPT3.5 for each method and dataset variable size. It is no surprise that PAL is the easiest to implement and, as a result, has the highest success rate. However, when it comes to Monte Carlo and ProbLog, we face a trade-off. Monte Carlo requires long code solutions that need precise order of ``if nesting'' to simulate a sampling function. On the other hand, ProbLog is easier to write but must be generated based on three in-context examples that explain the syntax and solution. Table~\ref{tab:ALLerros} shows that GPT3.5 struggles with ProbLog more regarding code generation between these two methods. This introduces another challenge with GPT3.5 that is not present with GPT4. ProbLog may be easier to write, but GPT3.5 is unfamiliar with it and struggles to write the correct code.

Llama3, on the other hand, achieves near-perfect (100\%) syntax accuracy on Python code, which reflects the data it is trained on. However, as shown in Table~\ref{tab:Llama3ProbLogALLerros}, Llama3's accuracy of ProbLog syntax is not the same. Some of these accuracies, such as $V_9$ and $V_{10}$, are the same as the accuracy of the main problem solution, implying that Llama3's lack of performance was mainly due to syntax errors.

\begin{table*}
\begin{center}
\small
\begin{tabular}{|c|c|c|c|c|c|c|c|c|c|}
\hline
Method & $V_2$ & $V_3$ & $V_4$ & $V_5$ & $V_6$ & $V_7$ & $V_8$ & $V_9$ & $V_{10}$ \\
\hline
ProbLog & 90 & 95 & 92 & 87 & 95 & 94 & 88 & 82 & 77 \\

\hline
\end{tabular}
\caption{Llama3 success rate of code execution for ProbLog as percentages.}
\label{tab:Llama3ProbLogALLerros}
\end{center}
\end{table*}

\section{Appendix E: Prompts}
\label{appendix:allrawPrompts}

\begin{figure}
    \centering
    \includegraphics[width=1.0\linewidth]{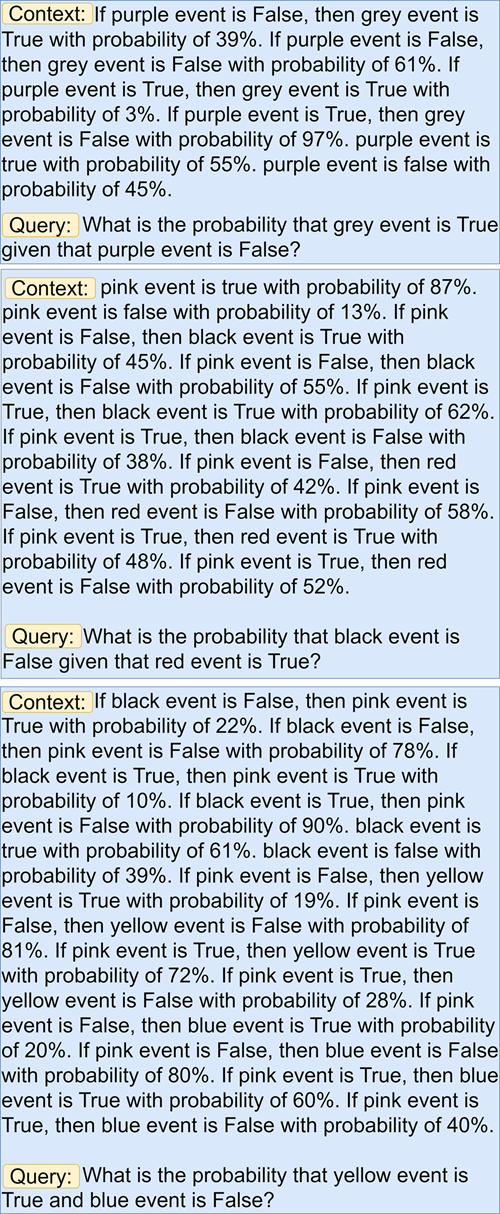}
    \caption{Contexts and Queries of in-context examples that were used in our prompts.}
    \label{fig:fsexamples}
\end{figure}

\begin{figure}
    \centering
    \includegraphics[width=0.8\linewidth]{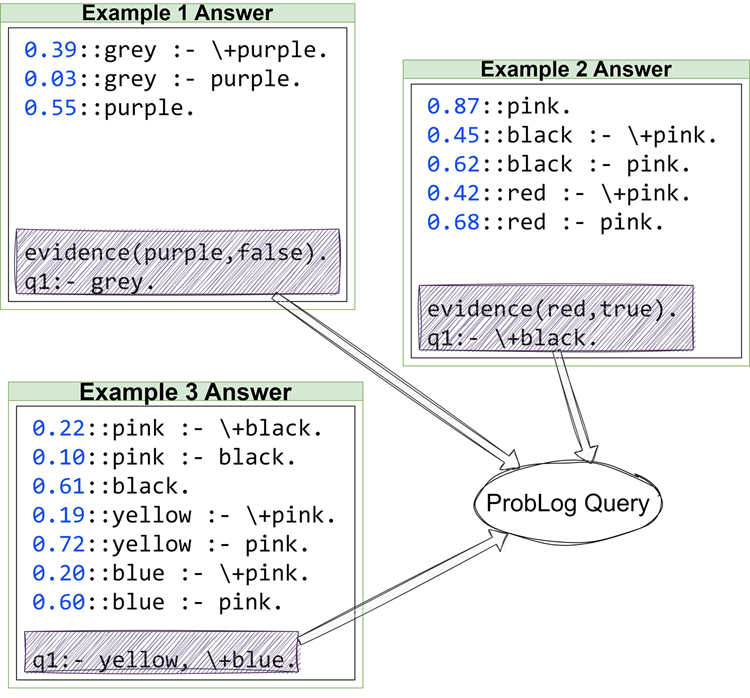}
    \caption{ProbLog answers to the in-context examples and their ``ProbLog Query''.}
    \label{fig:ProbLoganswers}
\end{figure}

In this section, we outline all the full prompts that were used in our LLMs for our selected examples, shown in Figure~\ref{fig:fsexamples}. This includes instructions, context, queries and answer for our methods Basic QA, COT , PAL , Monte Carlo and ProbLog. We also included the Prompts for our subtasks Number Extraction and Graph Generation which include instructions, contexts and answers ( query doesn't apply here). All our models (including open-source models) use the same instruction and examples. This prompts may require changes depending on whether we use subtasks without methods. For example if we don't include NE with PAL method, the variables in the answer will be replaced by actual numbers that appear in the context. If a subtask is used with a method, their instruction would also be merged.

The only time we changed the instruction for a specific model is Basic QA instruction for the Mistral method which did not follow the instruction in the zero-shot setting. In this scenario the Prompt changed to ``<s>[INST] You solve the following probabilistic question and generate the probability of the answer by only providing a number from 00 to 99 without any explanation. [/INST] [INST] Example : what is the probability of a coin flip?[/INST] 50 ''.

ProbLog needs special attention, given that the LLM may not be familiar with the programming language. In designing the ProbLog in-context prompts, selecting the examples to cover all the syntax and structure needed to write the ProbLog code for the rest of the examples is essential. Figure~\ref{fig:ProbLoganswers} shows the ProbLog answers to the in-context examples shown in Figure~\ref{fig:fsexamples} and their ``ProbLog Query''. These three ``ProbLog Queries'' showcase how to define a query for a target event (Examples 1 and 2), evidence event (Examples 1 and 2), and how to combine target events (Example 3). Also, in all the examples, multiple definitions of probabilities are given in the context that the LLM can emulate to write its code.

\subsection{Basic QA}
\label{appendix:EBQA}
\textbf{Instruction} You solve the following probabilistic question and generate the probability of the answer by only providing a number from 00 to 99 without any explanation.

\subsubsection{In-context Example 1 of Basic QA}
\noindent \textbf{Context} If purple event is False, then grey event is True with probability of 39\%. If purple event is False, then grey event is False with probability of 61\%. If purple event is True, then grey event is True with probability of 3\%. If purple event is True, then grey event is False with probability of 97\%. purple event is true with probability of 55\%. purple event is false with probability of 45\%.

\noindent \textbf{Question} What is the probability that grey event is True given that purple event is False?

\noindent \textbf{Answer} 39

\subsubsection{In-context Example 2 of Basic QA}
\noindent \textbf{Context} pink event is true with probability of 87\%. pink event is false with probability of 13\%. If pink event is False, then black event is True with probability of 45\%. If pink event is False, then black event is False with probability of 55\%. If pink event is True, then black event is True with probability of 62\%. If pink event is True, then black event is False with probability of 38\%. If pink event is False, then red event is True with probability of 42\%. If pink event is False, then red event is False with probability of 58\%. If pink event is True, then red event is True with probability of 48\%. If pink event is True, then red event is False with probability of 52\%.

\noindent \textbf{Question} What is the probability that black event is False given that red event is True?

\noindent \textbf{Answer} 40

\subsubsection{In-context Example 3 of Basic QA}
\noindent \textbf{Context} If black event is False, then pink event is True with probability of 22\%. If black event is False, then pink event is False with probability of 78\%. If black event is True, then pink event is True with probability of 10\%. If black event is True, then pink event is False with probability of 90\%. black event is true with probability of 61\%. black event is false with probability of 39\%. If pink event is False, then yellow event is True with probability of 19\%. If pink event is False, then yellow event is False with probability of 81\%. If pink event is True, then yellow event is True with probability of 72\%. If pink event is True, then yellow event is False with probability of 28\%. If pink event is False, then blue event is True with probability of 20\%. If pink event is False, then blue event is False with probability of 80\%. If pink event is True, then blue event is True with probability of 60\%. If pink event is True, then blue event is False with probability of 40\%.

\noindent \textbf{Question} What is the probability that yellow event is True given that blue event is False?

\noindent \textbf{Answer} 23

\subsection{COT}
\label{appendix:ECOT}
\textbf{Instruction} You solve the following probabilistic question and provide a detailed explanation of your solution and mathematical reasoning.
At the end of your answer, write a sentence that will give your final answer rounded to its closest integer, as such: The answer to this probabilistic question rounded to the closest integer is 71\%.

\subsubsection{In-context Example 1 of COT}
\noindent \textbf{Context} If purple event is False, then grey event is True with probability of 39\%. If purple event is False, then grey event is False with probability of 61\%. If purple event is True, then grey event is True with probability of 3\%. If purple event is True, then grey event is False with probability of 97\%. purple event is true with probability of 55\%. purple event is false with probability of 45\%.

\noindent \textbf{Question} What is the probability that grey event is True given that purple event is False?

\noindent \textbf{Answer} This probability is given to us in the context and as a result if final answer is 39\%.

\subsubsection{In-context Example 2 of COT}
\noindent \textbf{Context} pink event is true with probability of 87\%. pink event is false with probability of 13\%. If pink event is False, then black event is True with probability of 45\%. If pink event is False, then black event is False with probability of 55\%. If pink event is True, then black event is True with probability of 62\%. If pink event is True, then black event is False with probability of 38\%. If pink event is False, then red event is True with probability of 42\%. If pink event is False, then red event is False with probability of 58\%. If pink event is True, then red event is True with probability of 48\%. If pink event is True, then red event is False with probability of 52\%.

\noindent \textbf{Question} What is the probability that black event is False given that red event is True?

\noindent \textbf{Answer} To find: $P(\text{Black}=\text{False} \mid \text{Red}=\text{True})$

\noindent According to Bayes' Theorem, we have:
\begin{align*}
    P(&\text{Black}=\text{False} \mid \text{Red}=\text{True}) \\
    &= \frac{P(\text{Black}=\text{False} \text{ and } \text{Red}=\text{True})}{P(\text{Red}=\text{True})}.
\end{align*}

\noindent Calculate $P(\text{Black}=\text{False} \text{ and } \text{Red}=\text{True})$:
\begin{align*}
    P(&\text{Black}=\text{False} \text{ and } \text{Red}=\text{True}) \\
    &= P(\text{Black}=\text{False} \mid \text{Pink}=\text{True}) \\
    &\quad \times P(\text{Red}=\text{True} \mid \text{Pink}=\text{True}) \\
    &\quad \times P(\text{Pink}=\text{True}) \\
    &\quad + P(\text{Black}=\text{False} \mid \text{Pink}=\text{False}) \\
    &\quad \times P(\text{Red}=\text{True} \mid \text{Pink}=\text{False}) \\
    &\quad \times P(\text{Pink}=\text{False}) \\
    &= (0.38 \cdot 0.48 \cdot 0.87) \\
    &\quad + (0.55 \cdot 0.42 \cdot 0.13) \\
    &= 0.188718.
\end{align*}

\noindent Calculate $P(\text{Red}=\text{True})$:
\begin{align*}
    P(&\text{Red}=\text{True}) = \\
    &\quad P(\text{Red}=\text{True} \mid \text{Pink}=\text{True})\\
    &\quad \times P(\text{Pink}=\text{True})\\
    &\quad + P(\text{Red}=\text{True} \mid \text{Pink}=\text{False})\\
    &\quad \times P(\text{Pink}=\text{False}) \\
    &= (0.48 \cdot 0.87) + (0.42 \cdot 0.13) \\
    &= 0.4722.
\end{align*}

\noindent Calculate $P(\text{Black}=\text{False} \mid \text{Red}=\text{True})$:
\begin{align*}
    P(&\text{Black}=\text{False} \mid \text{Red}=\text{True}) \\
    &= \frac{0.188718}{0.4722} \\
    &= 0.3996.
\end{align*}

\noindent Thus, the final answer is approximately 40\%.

\subsubsection{In-context Example 3 of COT}
\noindent \textbf{Context} If black event is False, then pink event is True with probability of 22\%. If black event is False, then pink event is False with probability of 78\%. If black event is True, then pink event is True with probability of 10\%. If black event is True, then pink event is False with probability of 90\%. black event is true with probability of 61\%. black event is false with probability of 39\%. If pink event is False, then yellow event is True with probability of 19\%. If pink event is False, then yellow event is False with probability of 81\%. If pink event is True, then yellow event is True with probability of 72\%. If pink event is True, then yellow event is False with probability of 28\%. If pink event is False, then blue event is True with probability of 20\%. If pink event is False, then blue event is False with probability of 80\%. If pink event is True, then blue event is True with probability of 60\%. If pink event is True, then blue event is False with probability of 40\%.

\noindent \textbf{Question} What is the probability that yellow event is True given that blue event is False?

\noindent \textbf{Answer} 
\noindent We need to calculate the probability of Yellow event being True given that Blue event is False.

\begin{align*}
    P(&\text{Yellow}=\text{True} \mid \text{Blue}=\text{False}) =\\ 
    &\quad \frac{P(\text{Yellow}=\text{True} \text{and} \text{Blue}=\text{False})}{P(\text{Blue}=\text{False})}
\end{align*}

\noindent First, calculate $P(\text{Blue}=\text{False})$
\begin{align*}
P(&\text{Blue}=\text{False}) =\\
&\quad P(\text{Blue}=\text{False} \mid \text{Pink}=\text{True})\\
&\quad \times P(\text{Pink}=\text{True})\\
&\quad + P(\text{Blue}=\text{False} \mid \text{Pink}=\text{False})\\
&\quad \times P(\text{Pink}=\text{False})
\end{align*}

\noindent We need the probability of $P(\text{Pink}=\text{True})$ and $P(\text{Pink}=\text{False})$

\begin{align*}
P(&\text{Pink}=\text{True}) = \\
&\quad P(\text{Pink}=\text{True} \mid \text{Black}=\text{True}) \\
&\quad \times P(\text{Black}=\text{True})\\
&\quad + P(\text{Pink}=\text{True} \mid \text{Black}=\text{False}) \\
&\quad \times P(\text{Black}=\text{False}) \\
&= (0.1 \cdot 0.61) + (0.22 \cdot 0.39) = 0.1468 \\
P(&\text{Pink}=\text{False}) =\\
&\quad 1 - P(\text{Pink}=\text{True}) = 0.8532
\end{align*}

\noindent Now can calculate $P(\text{Blue}=\text{False})$
\begin{align*}
P(&\text{Blue}=\text{False}) = \\
&\quad P(\text{Blue}=\text{False} \mid \text{Pink}=\text{True})\\
&\quad \times P(\text{Pink}=\text{True})\\
&\quad + P(\text{Blue}=\text{False} \mid \text{Pink}=\text{False})\\
&\quad \times P(\text{Pink}=\text{False})\\
&= (0.4 \cdot 0.1468) + (0.8 \cdot 0.8532) = 0.74128
\end{align*}

\noindent Next, calculate P(\text{Yellow}=\text{True} \text{and} \text{Blue}=\text{False})
\begin{align*}
P(&\text{Yellow}=\text{True} \text{and} \text{Blue}=\text{False}) =\\
&\quad P(\text{Yellow}=\text{True} \mid \text{Pink}=\text{True})\\
&\quad \times P(\text{Blue}=\text{False} \mid \text{Pink}=\text{True})\\
&\quad \times P(\text{Pink}=\text{True})\\
&\quad + P(\text{Yellow}=\text{True} \mid \text{Pink}=\text{False})\\
&\quad \times P(\text{Blue}=\text{False} \mid \text{Pink}=\text{False})\\
&\quad \times P(\text{Pink}=\text{False})\\
&= (0.72 \cdot 0.4 \cdot 0.1468) + (0.19 \cdot 0.8 \cdot 0.8532)\\
&= 0.1719648
\end{align*}

\noindent Finally, calculate $P(\text{Yellow}=\text{True} \mid \text{Blue}=\text{False})$
\begin{align*}
P(&\text{Yellow}=\text{True} \mid \text{Blue}=\text{False}) \\
& = \frac{0.1719648}{0.74128} = 0.2319
\end{align*}
\noindent So the final answer is approximately 23\%.

\subsection{PAL}
\label{appendix:EPAL}
\textbf{Instruction} Solve the probabilistic questions by only writing a Python code. In Calculations section, the solution process is done, including the code and the mathematical reasoning in the form of Python code comments that start with \#. In Answer section, the final answer is assigned to a variable named answer.

\subsubsection{In-context Example 1 of PAL}
\noindent \textbf{Context} If purple event is False, then grey event is True with probability of 39\%. If purple event is False, then grey event is False with probability of 61\%. If purple event is True, then grey event is True with probability of 3\%. If purple event is True, then grey event is False with probability of 97\%. purple event is true with probability of 55\%. purple event is false with probability of 45\%.

\noindent \textbf{Question} What is the probability that grey event is True given that purple event is False?

\noindent \textbf{Answer} 
\begin{lstlisting}[language=Python]
# To find: P(Grey=True|Purple=False)
# This is given to us in the context

# Answer
answer=prob_grey_true_given_purple_false
\end{lstlisting}

\subsubsection{In-context Example 2 of PAL}
\noindent \textbf{Context} pink event is true with probability of 87\%. pink event is false with probability of 13\%. If pink event is False, then black event is True with probability of 45\%. If pink event is False, then black event is False with probability of 55\%. If pink event is True, then black event is True with probability of 62\%. If pink event is True, then black event is False with probability of 38\%. If pink event is False, then red event is True with probability of 42\%. If pink event is False, then red event is False with probability of 58\%. If pink event is True, then red event is True with probability of 48\%. If pink event is True, then red event is False with probability of 52\%.

\noindent \textbf{Question} What is the probability that black event is False given that red event is True?

\noindent \textbf{Answer} 
\begin{lstlisting}[language=Python]
# To find: P(Black=False|Red=True)
# According to Bayes' Theorem:
# P(Black=False|Red=True) = P(Black=False and Red=True) / P(Red=True)

# First, we find P(Black=False and Red=True)
# P(Black=False and Red=True) = P(Black=False|Pink=True)*P(Red=True|Pink=True)*P(Pink=True) + P(Black=False|Pink=False)*P(Red=True|Pink=False)*P(Pink=False)

prob_black_false_and_red_true = (prob_black_false_given_pink_true * prob_red_true_given_pink_true * prob_pink_true) + (prob_black_false_given_pink_false * prob_red_true_given_pink_false * prob_pink_false)

# Next, we find P(Red=True)
# P(Red=True) = P(Red=True|Pink=True)*P(Pink=True) + P(Red=True|Pink=False)*P(Pink=False)

prob_red_true = (prob_red_true_given_pink_true * prob_pink_true) + (prob_red_true_given_pink_false * prob_pink_false)

# Now, we can calculate P(Black=False|Red=True):
prob_black_false_given_red_true = prob_black_false_and_red_true / prob_red_true

# Answer
answer=prob_black_false_given_red_true
\end{lstlisting}

\subsubsection{In-context Example 3 of PAL}
\noindent \textbf{Context} If black event is False, then pink event is True with probability of 22\%. If black event is False, then pink event is False with probability of 78\%. If black event is True, then pink event is True with probability of 10\%. If black event is True, then pink event is False with probability of 90\%. black event is true with probability of 61\%. black event is false with probability of 39\%. If pink event is False, then yellow event is True with probability of 19\%. If pink event is False, then yellow event is False with probability of 81\%. If pink event is True, then yellow event is True with probability of 72\%. If pink event is True, then yellow event is False with probability of 28\%. If pink event is False, then blue event is True with probability of 20\%. If pink event is False, then blue event is False with probability of 80\%. If pink event is True, then blue event is True with probability of 60\%. If pink event is True, then blue event is False with probability of 40\%.

\noindent \textbf{Question} What is the probability that yellow event is True given that blue event is False?

\noindent \textbf{Answer} 
\begin{lstlisting}[language=Python]
# We need to calculate the probability of Yellow event being True given that Blue event is False.
# P(Yellow=True | Blue=False) = P(Yellow=True and Blue=False) / P(Blue=False)

# First, calculate P(Blue=False)
# P(Blue=False) = P(Blue=False | Pink=True)*P(Pink=True) + P(Blue=False | Pink=False)*P(Pink=False)

prob_pink_True = (prob_pink_true_given_black_true * prob_black_true) + (prob_pink_true_given_black_false * prob_black_false)
prob_pink_false = 1 - prob_pink_true

prob_blue_false = (prob_blue_false_given_pink_true * prob_pink_true) + (prob_blue_false_given_pink_false * prob_pink_false)

# Next, calculate P(Yellow=True and Blue=False)
# P(Yellow=True and Blue=False) = P(Yellow=True | Pink=True)*P(Blue=False | Pink=True)*P(Pink=True) +
#                                  P(Yellow=True | Pink=False)*P(Blue=False | Pink=False)*P(Pink=False)
prob_yellow_true_and_blue_false = (prob_yellow_true_given_pink_true * prob_blue_false_given_pink_true * prob_pink_true) + (prob_yellow_true_given_pink_false * prob_blue_false_given_pink_false * prob_pink_false)

# Finally, calculate P(Yellow=True | Blue=False)
prob_yellow_true_given_blue_false = prob_yellow_true_and_blue_false / prob_blue_false

# Answer
answer= prob_yellow_true_given_blue_false
\end{lstlisting}

\subsection{Monte Carlo}
\label{appendix:EMC}
\textbf{Instruction} Solve the probabilistic questions writing a Python code that would solve the question with a Monte Carlo method. The final probability is assigned to a variable named answer. Follow the format no matter what.

\subsubsection{In-context Example 1 of Monte Carlo}
\noindent \textbf{Context} If purple event is False, then grey event is True with probability of 39\%. If purple event is False, then grey event is False with probability of 61\%. If purple event is True, then grey event is True with probability of 3\%. If purple event is True, then grey event is False with probability of 97\%. purple event is true with probability of 55\%. purple event is false with probability of 45\%.

\noindent \textbf{Question} What is the probability that grey event is True given that purple event is False?

\noindent \textbf{Answer} 
\begin{lstlisting}[language=Python]
import random
def simulate():
    purple = random.random() < prob_purple_true
    if purple:
        grey = random.random() < prob_grey_true_given_purple_true
    else:
        grey = random.random() < prob_grey_true_given_purple_false
    return purple, grey

grey_true_given_purple_false_count = 0
purple_false_count = 0
for _ in range(1000000):
    purple, grey = simulate()
    if not purple:
        purple_false_count += 1
        if grey:
            grey_true_given_purple_false_count += 1

answer = grey_true_given_purple_false_count / purple_false_count if purple_false_count > 0 else 0
\end{lstlisting}

\subsubsection{In-context Example 2 of Monte Carlo}
\noindent \textbf{Context} pink event is true with probability of 87\%. pink event is false with probability of 13\%. If pink event is False, then black event is True with probability of 45\%. If pink event is False, then black event is False with probability of 55\%. If pink event is True, then black event is True with probability of 62\%. If pink event is True, then black event is False with probability of 38\%. If pink event is False, then red event is True with probability of 42\%. If pink event is False, then red event is False with probability of 58\%. If pink event is True, then red event is True with probability of 48\%. If pink event is True, then red event is False with probability of 52\%.

\noindent \textbf{Question} What is the probability that black event is False given that red event is True?

\noindent \textbf{Answer} 
\begin{lstlisting}[language=Python]
import random

def simulate():
    pink = random.random() < prob_pink_true
    if pink:
        black = random.random() < prob_black_true_given_pink_true
        red = random.random() < prob_red_true_given_pink_true
    else:
        black = random.random() < prob_black_true_given_pink_false
        red = random.random() < prob_red_true_given_pink_false
    return pink, black, red

black_false_given_red_true_count = 0
red_true_count = 0
for _ in range(1000000):
    pink, black, red = simulate()
    if red:
        red_true_count += 1
        if not black:
            black_false_given_red_true_count += 1

answer = black_false_given_red_true_count / red_true_count if red_true_count > 0 else 0
\end{lstlisting}

\subsubsection{In-context Example 3 of Monte Carlo}
\noindent \textbf{Context} If black event is False, then pink event is True with probability of 22\%. If black event is False, then pink event is False with probability of 78\%. If black event is True, then pink event is True with probability of 10\%. If black event is True, then pink event is False with probability of 90\%. black event is true with probability of 61\%. black event is false with probability of 39\%. If pink event is False, then yellow event is True with probability of 19\%. If pink event is False, then yellow event is False with probability of 81\%. If pink event is True, then yellow event is True with probability of 72\%. If pink event is True, then yellow event is False with probability of 28\%. If pink event is False, then blue event is True with probability of 20\%. If pink event is False, then blue event is False with probability of 80\%. If pink event is True, then blue event is True with probability of 60\%. If pink event is True, then blue event is False with probability of 40\%.

\noindent \textbf{Question} What is the probability that yellow event is True given that blue event is False?

\noindent \textbf{Answer} 
\begin{lstlisting}[language=Python]
import random
def simulate():
    black = random.random() < prob_black_true
    if black:
        pink = random.random() < prob_pink_true_given_black_true
    else:
        pink = random.random() < prob_pink_true_given_black_false
    if pink:
        yellow = random.random() < prob_yellow_true_given_pink_true
        blue = random.random() < prob_blue_true_given_pink_true
    else:
        yellow = random.random() < prob_yellow_true_given_pink_false
        blue = random.random() < prob_blue_true_given_pink_false
    return pink, yellow, blue

yellow_true_given_blue_false_count = 0
blue_false_count = 0
for _ in range(1000000):
    pink, yellow, blue = simulate()
    if not blue:
        blue_false_count += 1
        if yellow:
            yellow_true_given_blue_false_count += 1

answer = yellow_true_given_blue_false_count / blue_false_count if blue_false_count > 0 else 0
\end{lstlisting}

\subsection{ProbLog}
\label{appendix:EProbLog}
\textbf{Instruction} You solve probabilistic questions by writing a Problog code that represents the probabilities, and the query.

\subsubsection{In-context Example 1 of ProbLog}
\noindent \textbf{Context} If purple event is False, then grey event is True with probability of 39\%. If purple event is False, then grey event is False with probability of 61\%. If purple event is True, then grey event is True with probability of 3\%. If purple event is True, then grey event is False with probability of 97\%. purple event is true with probability of 55\%. purple event is false with probability of 45\%.

\noindent \textbf{Question} What is the probability that grey event is True given that purple event is False?

\noindent \textbf{Answer} 
\begin{lstlisting}[language=Python]
{prob_grey_true_given_purple_false}::grey :- \+purple.
{prob_grey_true_given_purple_true}::grey :- purple.
{prob_purple_true}::purple.

evidence(purple,false).
q1:- grey.
\end{lstlisting}

\subsubsection{In-context Example 2 of ProbLog}
\noindent \textbf{Context} pink event is true with probability of 87\%. pink event is false with probability of 13\%. If pink event is False, then black event is True with probability of 45\%. If pink event is False, then black event is False with probability of 55\%. If pink event is True, then black event is True with probability of 62\%. If pink event is True, then black event is False with probability of 38\%. If pink event is False, then red event is True with probability of 42\%. If pink event is False, then red event is False with probability of 58\%. If pink event is True, then red event is True with probability of 48\%. If pink event is True, then red event is False with probability of 52\%.

\noindent \textbf{Question} What is the probability that black event is False given that red event is True?

\noindent \textbf{Answer} 
\begin{lstlisting}[language=Python]
{prob_pink_true}::pink.
{prob_black_true_given_pink_false}::black :- \+pink.
{prob_black_true_given_pink_true}::black :- pink.
{prob_red_true_given_pink_false}::red :- \+pink.
{prob_red_true_given_pink_true}::red :- pink.

evidence(red,true).
q1:- \+black.
\end{lstlisting}

\subsubsection{In-context Example 3 of ProbLog}
\noindent \textbf{Context} If black event is False, then pink event is True with probability of 22\%. If black event is False, then pink event is False with probability of 78\%. If black event is True, then pink event is True with probability of 10\%. If black event is True, then pink event is False with probability of 90\%. black event is true with probability of 61\%. black event is false with probability of 39\%. If pink event is False, then yellow event is True with probability of 19\%. If pink event is False, then yellow event is False with probability of 81\%. If pink event is True, then yellow event is True with probability of 72\%. If pink event is True, then yellow event is False with probability of 28\%. If pink event is False, then blue event is True with probability of 20\%. If pink event is False, then blue event is False with probability of 80\%. If pink event is True, then blue event is True with probability of 60\%. If pink event is True, then blue event is False with probability of 40\%.

\noindent \textbf{Question} What is the probability that yellow event is True given that blue event is False?

\noindent \textbf{Answer} 
\begin{lstlisting}[language=Python]
{prob_pink_true_given_black_false}::pink :- \+black.
{prob_pink_true_given_black_true}::pink :- black.
{prob_black_true}::black.
{prob_yellow_true_given_pink_false}::yellow :- \+pink.
{prob_yellow_true_given_pink_true}::yellow :- pink.
{prob_blue_true_given_pink_false}::blue :- \+pink.
{prob_blue_true_given_pink_true}::blue :- pink.

q1:- yellow, \+blue.
\end{lstlisting}

\subsection{NE Subtask}
\label{appendix:ENE}
\textbf{Instruction} Extract probabilities in the context in the format of given examples.

\subsubsection{In-context Example 1 of NE}
\noindent \textbf{Context} If purple event is False, then grey event is True with probability of 39\%. If purple event is False, then grey event is False with probability of 61\%. If purple event is True, then grey event is True with probability of 3\%. If purple event is True, then grey event is False with probability of 97\%. purple event is true with probability of 55\%. purple event is false with probability of 45\%.

\noindent \textbf{Answer}
\begin{lstlisting}[language=Python]
prob_grey_true_given_purple_false = 0.39
prob_grey_false_given_purple_false = 0.61
prob_grey_true_given_purple_true = 0.03
prob_grey_false_given_purple_true = 0.97
prob_purple_true = 0.55
prob_purple_false = 0.45
\end{lstlisting}

\subsubsection{In-context Example 2 of NE}
\noindent \textbf{Context} pink event is true with probability of 87\%. pink event is false with probability of 13\%. If pink event is False, then black event is True with probability of 45\%. If pink event is False, then black event is False with probability of 55\%. If pink event is True, then black event is True with probability of 62\%. If pink event is True, then black event is False with probability of 38\%. If pink event is False, then red event is True with probability of 42\%. If pink event is False, then red event is False with probability of 58\%. If pink event is True, then red event is True with probability of 48\%. If pink event is True, then red event is False with probability of 52\%.

\noindent \textbf{Answer}
\begin{lstlisting}[language=Python]
prob_pink_true = 0.87
prob_pink_false = 0.13
prob_black_true_given_pink_false = 0.45
prob_black_false_given_pink_false = 0.55
prob_black_true_given_pink_true = 0.62
prob_black_false_given_pink_true = 0.38
prob_red_true_given_pink_false = 0.42
prob_red_false_given_pink_false = 0.58
prob_red_true_given_pink_true = 0.48
prob_red_false_given_pink_true = 0.52
\end{lstlisting}

\subsubsection{In-context Example 3 of NE}
\noindent \textbf{Context} If black event is False, then pink event is True with probability of 22\%. If black event is False, then pink event is False with probability of 78\%. If black event is True, then pink event is True with probability of 10\%. If black event is True, then pink event is False with probability of 90\%. black event is true with probability of 61\%. black event is false with probability of 39\%. If pink event is False, then yellow event is True with probability of 19\%. If pink event is False, then yellow event is False with probability of 81\%. If pink event is True, then yellow event is True with probability of 72\%. If pink event is True, then yellow event is False with probability of 28\%. If pink event is False, then blue event is True with probability of 20\%. If pink event is False, then blue event is False with probability of 80\%. If pink event is True, then blue event is True with probability of 60\%. If pink event is True, then blue event is False with probability of 40\%.

\noindent \textbf{Answer}
\begin{lstlisting}[language=Python]
prob_pink_true_given_black_false = 0.22
prob_pink_false_given_black_false = 0.78
prob_pink_true_given_black_true = 0.10
prob_pink_false_given_black_true = 0.90
prob_black_true = 0.61
prob_black_false = 0.39
prob_yellow_true_given_pink_false = 0.19
prob_yellow_false_given_pink_false = 0.81
prob_yellow_true_given_pink_true = 0.72
prob_yellow_false_given_pink_true = 0.28
prob_blue_true_given_pink_false = 0.20
prob_blue_false_given_pink_false = 0.80
prob_blue_true_given_pink_true = 0.60
prob_blue_false_given_pink_true = 0.40
\end{lstlisting}

\subsection{GG Subtask}
\label{appendix:EGG}
\textbf{Instruction} Extract the graph edges of the described Bayesian network in the context in the format of given examples.

\subsubsection{In-context Example 1 of GG}
\noindent \textbf{Context} If purple event is False, then grey event is True with probability of 39\%. If purple event is False, then grey event is False with probability of 61\%. If purple event is True, then grey event is True with probability of 3\%. If purple event is True, then grey event is False with probability of 97\%. purple event is true with probability of 55\%. purple event is false with probability of 45\%.

\noindent \textbf{Answer} purple event $\implies$ grey event

\subsubsection{In-context Example 2 of GG}
\noindent \textbf{Context} pink event is true with probability of 87\%. pink event is false with probability of 13\%. If pink event is False, then black event is True with probability of 45\%. If pink event is False, then black event is False with probability of 55\%. If pink event is True, then black event is True with probability of 62\%. If pink event is True, then black event is False with probability of 38\%. If pink event is False, then red event is True with probability of 42\%. If pink event is False, then red event is False with probability of 58\%. If pink event is True, then red event is True with probability of 48\%. If pink event is True, then red event is False with probability of 52\%.

\noindent \textbf{Answer} pink event $\implies$ black event, pink event $\implies$ red event

\subsubsection{In-context Example 3 of GG}
\noindent \textbf{Context} If black event is False, then pink event is True with probability of 22\%. If black event is False, then pink event is False with probability of 78\%. If black event is True, then pink event is True with probability of 10\%. If black event is True, then pink event is False with probability of 90\%. black event is true with probability of 61\%. black event is false with probability of 39\%. If pink event is False, then yellow event is True with probability of 19\%. If pink event is False, then yellow event is False with probability of 81\%. If pink event is True, then yellow event is True with probability of 72\%. If pink event is True, then yellow event is False with probability of 28\%. If pink event is False, then blue event is True with probability of 20\%. If pink event is False, then blue event is False with probability of 80\%. If pink event is True, then blue event is True with probability of 60\%. If pink event is True, then blue event is False with probability of 40\%.

\noindent \textbf{Answer} black event $\implies$ pink event, pink event $\implies$ yellow event, pink event $\implies$ blue event

\end{document}